\documentclass[11pt]{article}

\usepackage[preprint]{acl}

\usepackage{times}
\usepackage{latexsym}

\usepackage[T1]{fontenc}

\usepackage[utf8]{inputenc}

\usepackage{microtype}
\usepackage{inconsolata}

\usepackage[utf8]{inputenc} 
\usepackage[T1]{fontenc}    
\usepackage{hyperref}       
\usepackage{url}            
\usepackage{booktabs}       
\usepackage{amsfonts}       
\usepackage{nicefrac}       
\usepackage{microtype}      
\usepackage{xcolor}         
\usepackage{amsmath}
\usepackage{amssymb}
\usepackage{mathtools}
\usepackage{amsthm}
\usepackage{multirow}
\usepackage{graphicx}
\usepackage[most]{tcolorbox}
\usepackage{listings}
\usepackage{enumitem}
\usepackage{wrapfig}
\usepackage{algorithm}
\usepackage{algorithmic}
\usepackage{textcomp}

\usepackage{graphicx}

\usepackage{amsmath}
\usepackage{amssymb}
\usepackage{amsthm}

\theoremstyle{plain}
\newtheorem{theorem}{Theorem}
\newtheorem{proposition}[theorem]{Proposition}

\theoremstyle{definition}
\newtheorem{definition}[theorem]{Definition}
\newtheorem{assumption}[theorem]{Assumption}
\theoremstyle{remark}

\title{Share More, Search Less: Collaborative Parallel Thinking for Efficient Test-Time Scaling}
\begin{document}

\renewcommand{\thefootnote}{\fnsymbol{footnote}}
\author{
    \textbf{Xinglin Wang}\textsuperscript{\rm 1}\footnotemark[1], \hspace{0cm}
    \textbf{Hao Lin}\textsuperscript{\rm 1}\footnotemark[1], \hspace{0cm}
    \textbf{Shaoxiong Feng}\textsuperscript{\rm 2}\footnotemark[2], \hspace{0cm}
    \textbf{Peiwen Yuan}\textsuperscript{\rm 1}, \hspace{0cm}
    \textbf{Yiwei Li}\textsuperscript{\rm 1}, \hspace{0cm}
    \textbf{Jiayi Shi}\textsuperscript{\rm 1}, \hspace{0cm} \\
    \textbf{Yueqi Zhang}\textsuperscript{\rm 1}, \hspace{0cm}
    \textbf{Chuyi Tan}\textsuperscript{\rm 1}, \hspace{0cm}
    \textbf{Ji Zhang}\textsuperscript{\rm 1}, \hspace{0cm}
    \textbf{Boyuan Pan}\textsuperscript{\rm 2}, \hspace{0cm}
    \textbf{Yao Hu}\textsuperscript{\rm 2}, \hspace{0cm}
    \textbf{Kan Li}\textsuperscript{\rm 1}\footnotemark[2] \\
    \textsuperscript{\rm 1} School of Computer Science, Beijing Institute of Technology \\
    \textsuperscript{\rm 2} Xiaohongshu Inc \\
    \texttt{\{wangxinglin,linhao,peiwenyuan,liyiwei\}@bit.edu.cn} \\
    \texttt{\{shijiayi, zhangyq,tanchuyi,zhangji,likan\}@bit.edu.cn} \\
    \texttt{shaoxiongfeng2023@gmail.com} \quad
    \texttt{\{panboyuan,xiahou\}@xiaohongshu.com}
}
\maketitle

\footnotetext[1]{Equal contribution.}
\footnotetext[2]{Corresponding author.}

\renewcommand{\thefootnote}{\arabic{footnote}}

\begin{abstract}
Test-Time Scaling (TTS) enhances the reasoning capabilities of large language models by allocating additional inference compute to explore the solution space. However, existing parallel TTS methods typically keep branches isolated during search: intermediate discoveries remain branch-private and cannot guide other branches in time. This information isolation causes substantial redundant exploration, as branches repeatedly rediscover information already found elsewhere and require more search steps to collect complete decision information needed to reach correct answers. To bridge this gap, we propose \textbf{Collaborative Parallel Thinking (CPT)}, a training-free inference framework that enables search-time information sharing across parallel branches. CPT extracts compact intermediate information from ongoing branches, maintains a deduplicated query-level information pool, and broadcasts pool entries through the input context, allowing each branch in subsequent search steps to reuse discoveries made by other branches rather than rediscover the same information. Empirically, experiments on HMMT and AIME benchmarks show that CPT establishes a stronger accuracy--latency Pareto frontier than strong baselines across rollout budgets and model scales, highlighting search-time collaboration as an effective direction for efficient parallel TTS\footnote{Our code and data have been released on \url{https://github.com/WangXinglin/CPT}.}.
\end{abstract}


\section{Introduction}

Test-time scaling (TTS) has emerged as a powerful paradigm for improving the reasoning capabilities of large language models~\citep{brown2024large,snell2024scaling,wu2025inference,zhang2025survey,liu2025can}.
By allocating additional computation at inference time, TTS enables extended exploration of the solution space, allowing models to gather more decision-relevant information, reduce uncertainty about the final answer, and thereby improve performance on complex tasks such as mathematical reasoning~\citep{o1, Kimi-k1.5,DeepSeek-R1}.
Rather than concentrating exploration in a single trajectory, parallel TTS generates multiple reasoning branches for the same problem simultaneously, accelerating solution-space exploration~\citep{self-consistency, lightman2023let, baseline-deep-c, zheng2026parallel,zheng2026llms}.

\begin{figure}[t]
\centering
\includegraphics[width=\linewidth]{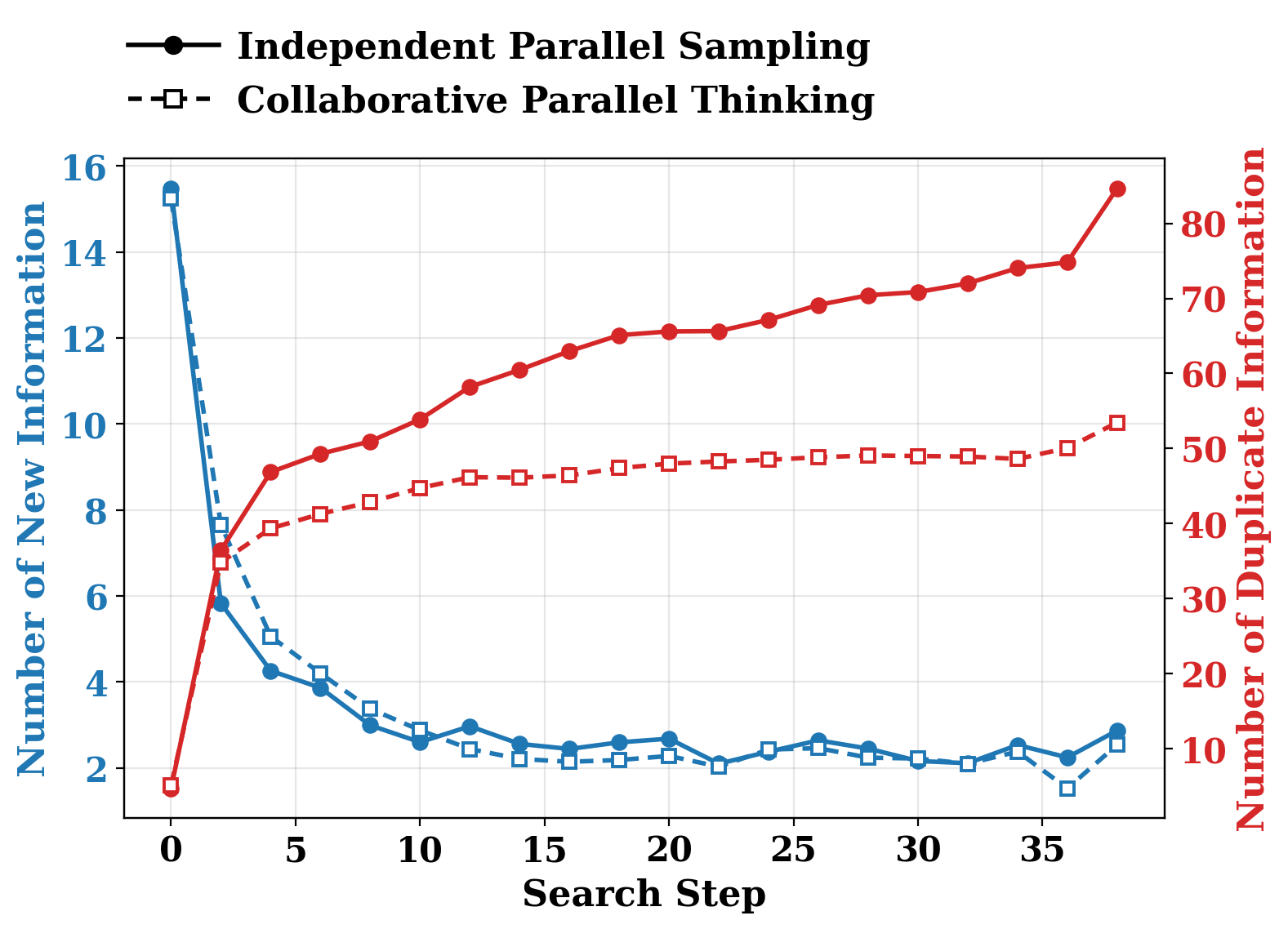}
\caption{
Information statistics during parallel search on HMMT24--25 using \textsc{Qwen3-30B-A3B-Thinking-2507} with 64 parallel samples.
At each 1024-token search step, we count newly discovered and duplicate information units, with details provided in Appendix~\ref{app:info_statistics}.
As search proceeds, independent parallel sampling yields more duplicate information units per step, whereas Collaborative Parallel Thinking mitigates this redundancy while maintaining a comparable number of newly discovered information units.
}
\label{fig:gain_duplication}
\end{figure}

\begin{table}[t]
\centering
\small
\setlength{\tabcolsep}{2.85pt}
\begin{tabular}{lcccccc}
\toprule
\multirow{2}{*}{\textbf{Metrics}} 
& \multicolumn{6}{c}{\textbf{Injection Ratio of Information (\%)}} \\ 
\cmidrule{2-7}
& \textbf{0} & \textbf{20} & \textbf{40} & \textbf{60} & \textbf{80} & \textbf{100} \\ 
\midrule
Pass@1 (\%) & 48.95 & 53.34 & 54.82 & 55.89 & 56.20 & 55.86 \\
Tokens      & 26714 & 16805 & 14128 & 12792 & 11924 & 11412 \\
Latency (s)     & 395   & 269   & 235   & 223   & 219   & 220   \\
\bottomrule
\end{tabular}
\caption{
Effect of sharing branch-distributed information via offline injection on HMMT24--25 using \textsc{Qwen3-4B-Thinking-2507} with 64 parallel samples.
Each injection ratio corresponds to randomly sampling that proportion of entries from a deduplicated information pool extracted from parallel branches and injecting them into the initial context (details in Appendix~\ref{app:info_injection}).
}
\label{tab:info_injection_4b}
\end{table}

\begin{figure*}[t]
\centering
\includegraphics[width=0.96\textwidth]{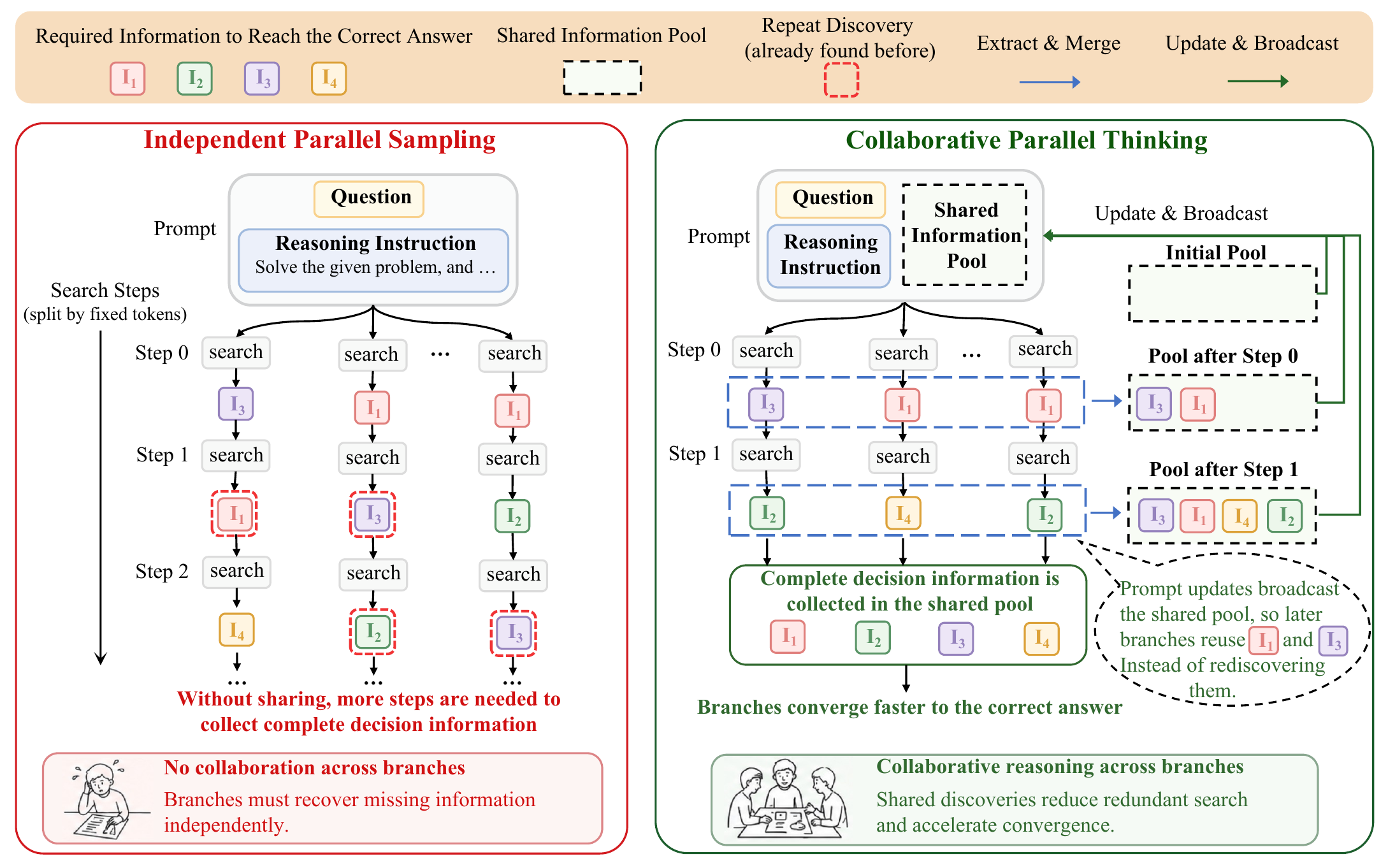}
\caption{
\textbf{Search-time collaboration reduces redundant parallel exploration.}
\textbf{(Left)} Independent parallel sampling keeps branches isolated during search, so each branch must recover missing decision information independently and may repeatedly rediscover information found elsewhere.
\textbf{(Right)} Collaborative Parallel Thinking (CPT) maintains a shared information pool that is provided to all branches. At fixed-token search steps, branch-level information is extracted, merged, and deduplicated into this pool. When broadcasting is active, pool entries are provided through the input context, allowing branches to reuse shared discoveries and converge faster once complete decision information is collected.}
\label{fig:overview}
\end{figure*}

While yielding promising performance gains, existing parallel TTS methods share a common bottleneck in information utilization~\citep{wang2025survey}: branch-level discoveries remain confined to the trajectories that produce them, rather than being converted into shared decision information during search (Figure~\ref{fig:overview}).
Consequently, different branches may repeatedly rediscover intermediate information that has already been found elsewhere, causing redundant exploration and slowing the accumulation of global decision information.

To validate this information-isolation bottleneck, we conduct two preliminary studies that examine both its induced redundancy and the utility of information that remains unshared across branches during search.
To quantify the redundancy induced by information isolation, we track information statistics at each fixed-token search step by counting newly discovered information units and duplicate information units.
Figure~\ref{fig:gain_duplication} shows that, as search progresses, independent parallel sampling repeats information already discovered by other branches, wasting substantial computation for only diminishing marginal information gain.
We further ask whether branch-distributed discoveries are useful once shared beyond the branches that originally found them.
As shown in Table~\ref{tab:info_injection_4b}, offline injection of information sampled from a deduplicated branch-information pool supports more accurate reasoning and faster search convergence.
Together, these results indicate that the bottleneck is not a lack of useful intermediate information within individual branches, but \emph{the absence of a search-time mechanism for pooling and broadcasting branch-private discoveries as reusable global decision information}.


Motivated by this bottleneck, we propose \textbf{Collaborative Parallel Thinking (CPT)}, a training-free inference method that enables parallel reasoning branches to share and reuse decision-relevant information during search.
Rather than treating branch-level discoveries as private byproducts of independent branches, CPT extracts compact information from ongoing branches and maintains it in a deduplicated query-level pool.
At fixed-token search steps, pool entries are broadcast through the input context, so subsequent generation chunks can reuse discoveries made by other branches while preserving each branch's private reasoning history.
To avoid premature convergence of exploration directions, CPT further adopts an adaptive broadcast schedule that controls when to share: branches first explore independently, broadcasting starts after marginal new-information gain drops substantially, and synchronization stops once further sharing offers limited benefit.



We evaluate CPT on challenging mathematical reasoning benchmarks, including HMMT24, HMMT25~\citep{balunovic_srimatharena_2025}, AIME24, AIME25, and AIME26 \citep{AIME26}, across a broad range of rollout budgets and policy models.
The empirical results show that CPT establishes a stronger accuracy--latency Pareto frontier than strong baselines across rollout budgets, achieving higher accuracy at matched wall-clock latency and demonstrating the efficiency gains of reducing redundant exploration through search-time collaboration.

To summarize, our main contributions are:
\begin{itemize}[leftmargin=*]
    \item We identify and validate an information-isolation bottleneck in parallel TTS: branch-private discoveries remain unshared during search, causing redundant rediscovery and slowing convergence toward correct answers.
    
    \item We propose Collaborative Parallel Thinking (CPT), a training-free inference method that enables parallel branches to share and reuse decision-relevant information during search through compact extraction, deduplicated pooling, shared-context broadcasting, and adaptive broadcast scheduling.
    
    \item We empirically validate CPT on challenging mathematical reasoning benchmarks, including HMMT24, HMMT25, AIME24, AIME25, and AIME26, demonstrating a stronger accuracy--latency Pareto frontier than strong baselines across diverse rollout budgets and model scales.
\end{itemize}

\section{Related Work}

\paragraph{Test-Time Scaling.}
The paradigm of scaling test-time compute has emerged as a critical avenue for enhancing reasoning capabilities~\citep{o1,DeepSeek-R1,wu2025inference}.
Existing methods differ in how this compute is organized~\citep{zhang2025survey}.
Sequential refinement methods extend or revise a single reasoning trajectory over time~\citep{wei2022chain,Self-Refine,shinn2023reflexion}, while structured search methods such as Tree-of-Thoughts, MCTS-style decoding, and verifier-guided search introduce explicit branching, scoring, and pruning over candidate prefixes~\citep{yao2023tree,liu2023don,wan2024alphazero,lightman2023let,Wang2025EveryRC}.
These methods improve the quality of test-time search by steering computation toward more promising reasoning states, but their reliance on sequential refinement, intermediate evaluation, or verifier-guided control can lengthen the critical inference path and introduce additional latency overhead~\citep{snell2024scaling,wang2025faster,zheng2026parallel,zheng2026llms,wang2026not}.
Parallel TTS instead generates multiple reasoning branches for the same problem simultaneously, accelerating solution-space exploration through width scaling~\citep{self-consistency,ESC,baseline-deep-c,zheng2025parallel}.
However, standard parallel sampling typically keeps branches independent during search and aggregates their outputs only after generation, leaving intermediate discoveries isolated within individual branches.
The closest attempt to move beyond this independent formulation is LeaP~\citep{luo2025learning}, which introduces peer-path routing for self-correction and helps reasoning paths recover from poor prefixes during inference.
Unlike LeaP's path-specific routing for self-correction, CPT is efficiency-oriented: it maintains a deduplicated query-level shared information pool and broadcasts it globally through the input context, turning branch-private discoveries into reusable decision information to reduce redundant rediscovery.

\paragraph{Collaborative Reasoning with LLMs.}
Collaboration among LLM instances has been explored through multi-agent debate and discussion, where agents exchange arguments, critiques, or proposals to encourage divergent reasoning, consensus formation, or improved final decisions~\citep{liang2024encouraging,chen2024reconcile,estornell2024multi,liu2025breaking}.
Recent analyses further examine when such interaction is beneficial compared with single-agent prompting or voting-based aggregation~\citep{wang2024rethinking,choi2026debate}.
To improve efficiency, follow-up work studies sparse communication topology, token-efficient debate, group discussion, and selective debate triggering~\citep{li2024improving,zeng2025s2,liu2024groupdebate,eo2025debate}.
Beyond debate, Mixture-of-Agents and scalable multi-agent collaboration frameworks aggregate or orchestrate multiple model instances to exploit complementary strengths~\citep{wang2025mixture,qian2025scaling}.
These works organize collaboration mainly at the agent, model, or response level.
In contrast, CPT studies collaboration among parallel reasoning branches of the same query during test-time search, sharing compact intermediate information through a deduplicated query-level pool to reduce redundant rediscovery.

\section{Methodology}
\label{sec:method}

Collaborative Parallel Thinking (CPT) is motivated by a simple principle: parallel branches should share useful discoveries during search instead of repeatedly rediscovering information found elsewhere.
We therefore turn parallel TTS from isolated branch-wise exploration into a search-time information-sharing process, where compact discoveries are extracted into a deduplicated query-level pool and broadcast through the input context to guide subsequent exploration.
To keep sharing lightweight and stable, CPT synchronizes branches only at fixed-token search steps while preserving their private reasoning histories.

The overall workflow of CPT comprises three coordinated components:
\textbf{Collaborative Parallel Search}, which synchronizes branches at fixed-token steps while preserving their private reasoning histories;
\textbf{Shared Information Pooling}, which extracts compact information, maintains a deduplicated query-level pool, and broadcasts pool entries through the input context; and
\textbf{Adaptive Broadcast Scheduling}, which decides when to start and stop broadcasting according to the relative rate of newly discovered information.
A detailed algorithmic description is provided in Algorithm~\ref{alg:cpt} in Appendix~\ref{app:cpt_algorithm}.

\subsection{Collaborative Parallel Search}
\label{sec:cpt_search}

Given a problem \(x\) and a policy model \(\pi\), CPT launches \(K\) reasoning branches in parallel.
At search step \(t\), each branch \(i\) maintains its own private reasoning history \(h_i^t\), while the system maintains a query-level information pool \(\mathcal{P}_t\).
Each search step corresponds to a fixed-token generation chunk of length \(C\).
During a step, every unfinished branch continues generation from its own private history, optionally conditioned on the current shared information block if broadcasting is active.
After the chunk is generated, CPT temporarily synchronizes the branches and passes their newly generated segments to the information-pooling stage.

A key design choice is that CPT only synchronizes at step boundaries.
It does not rewrite the generated reasoning trace, nor does it inject shared content into the middle of a branch's private chain.
Thus, each branch preserves its own exploration trajectory, while the shared information block serves as an additional context for subsequent search steps.
This design preserves branch-level exploration while enabling search-time reuse of information discovered by other branches.

\subsection{Shared Information Pooling}
\label{sec:pool_broadcast}

Raw reasoning trajectories are too long and noisy to be shared directly.
CPT therefore asks the same policy model \(\pi\) to extract compact search information from each branch, such as intermediate conclusions, constraints, observations, counterexamples, and useful checks.
Given the problem \(x\), the current branch history \(h_i^t\), and the newly generated segment \(\Delta h_i^t\), the extractor outputs a small set of textual information units:
\[
\mathcal{Z}_i^t = \mathsf{Extract}_{\pi}(x,h_i^t,\Delta h_i^t).
\]
This extraction step converts unstructured reasoning traces into compact information units that can be pooled and broadcast under a limited context budget.

All extracted information units are maintained in a deduplicated query-level pool \(\mathcal{P}_t\).
When updating the pool, CPT filters out candidates that are semantically similar to existing entries under an embedding-similarity threshold \(\tau_{\mathrm{dup}}\), where similarities are computed with text embeddings \(\phi(\cdot)\).
This keeps the pool compact and information-dense, preventing repeated branch discoveries from consuming the shared context budget.

When broadcasting is active, CPT samples at most \(M\) entries from the deduplicated pool and serializes them into the shared-information section of the input context.
Before the next search step, the same updated context is provided to all branches, while each branch keeps its own private reasoning history.
In this way, branch-private discoveries become reusable global decision information that can guide subsequent parallel exploration.

\subsection{Adaptive Broadcast Scheduling}
\label{sec:broadcast_schedule}

Broadcasting is helpful only when the shared pool contains information that can influence later search.
Starting too early may reduce exploration diversity before enough local discoveries have been collected, while continuing too late brings little new information but still incurs synchronization and prefilling overhead.
CPT therefore uses the rate of newly admitted pool entries as a lightweight proxy for marginal information gain.

Let \(n_t\) be the number of new entries added to the pool at search step \(t\).
To make this signal stable, CPT averages it over fixed windows of \(W\) search steps.
For the \(j\)-th window \(\mathcal{W}_j\), we compute
\[
g_j = \frac{1}{|\mathcal{W}_j|}\sum_{t\in \mathcal{W}_j} n_t,
\qquad
r_j = \frac{g_j}{g_1+\epsilon},
\]
where \(g_1\) is the average new-information count in the first window and \(\epsilon\) is a small constant for numerical stability.
The ratio \(r_j\) measures how much the current marginal information gain has declined relative to the initial exploration stage.
Using the first window as a problem-specific reference allows CPT to adapt to different reasoning lengths and information densities.

CPT uses two thresholds, a start threshold \(\tau_{\mathrm{start}}\) and a stop threshold \(\tau_{\mathrm{stop}}\), with \(\tau_{\mathrm{stop}}<\tau_{\mathrm{start}}\), to control the broadcast schedule.
In the \textbf{probe phase}, branches explore independently while still writing extracted information into the pool.
When \(r_j<\tau_{\mathrm{start}}\), CPT enters the \textbf{broadcast phase} and begins providing the shared pool to all branches at subsequent search steps.
When \(r_j<\tau_{\mathrm{stop}}\), CPT enters the \textbf{free-run phase}, stops further synchronization, and lets each branch decode until termination.
This schedule activates sharing when independent exploration starts to yield fewer new discoveries, and stops sharing when the marginal benefit of additional broadcasts becomes low.

\section{Experiments}
\label{sec:experiments}

\subsection{Experimental Setup}
\label{sec:exp_setup}

\paragraph{Benchmarks and Models.}
We evaluate the proposed CPT strategy on challenging mathematical reasoning benchmarks, including HMMT24, HMMT25~\citep{balunovic_srimatharena_2025}, AIME25, and AIME26~\citep{AIME26}.
These benchmarks require complex multi-step reasoning and serve as reliable testbeds for evaluating whether test-time methods can efficiently explore solution spaces under practical latency budgets.
We primarily evaluate CPT on models specialized for complex reasoning tasks, including \textsc{Qwen3-4B-Thinking-2507} and \textsc{Qwen3-30B-A3B-Thinking-2507}~\citep{llm-qwen3}.

\paragraph{Baselines.}
We compare CPT against representative parallel test-time scaling and efficiency strategies:
(1) \textbf{Base Parallel Sampling}, which directly samples multiple reasoning branches from the policy model without any additional search-time mechanism, reflecting the model's raw parallel-sampling capability;
(2) \textbf{DeepConf}~\citep{baseline-deep-c}, which uses confidence signals to identify and retain high-quality reasoning traces;
and
(3) \textbf{LeaP}~\citep{luo2025learning}, the closest cross-path interaction baseline, which uses peer-path routing for self-correction during inference.
Since our focus is latency-efficient parallel TTS, we do not compare against sequential or structured search methods such as Self-Refine, Tree-of-Thoughts, and MCTS-style search, whose iterative refinement or lookahead procedures often require additional serial computation and introduce substantial end-to-end latency overhead.

\begin{figure*}[t]
\centering
\includegraphics[width=1.0\textwidth]{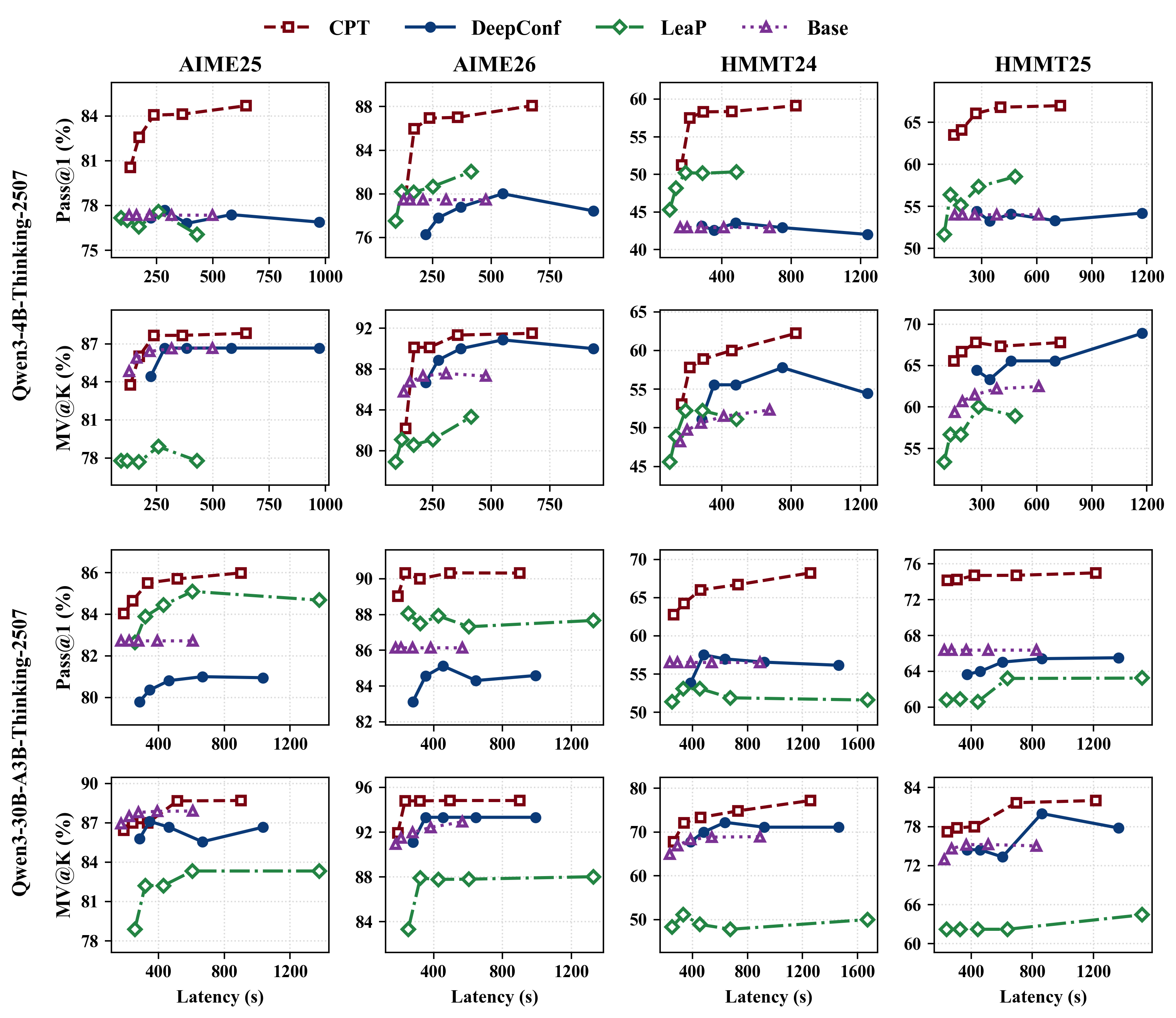}
\caption{Accuracy--Latency comparison across models and benchmarks under different rollout budgets.
}
\label{fig:main_results}
\end{figure*}

\paragraph{Evaluation Protocol.}
Unless otherwise specified, we evaluate each method under a range of parallel rollout budgets \(K\in\{8,16,32,64,128\}\), where larger \(K\) corresponds to scaling test-time computation by launching more reasoning branches.
For each budget, we measure wall-clock latency from the start of parallel decoding until all reasoning branches finish decoding, and report two accuracy metrics after answer normalization.
\textbf{Pass@1} measures the average correctness of individual sampled branches, while \textbf{MV@K} measures the majority-vote accuracy obtained by aggregating the final answers produced under the corresponding \(K\)-branch budget.
To ensure statistical reliability, we conduct \(8\) independent runs for all experiments, and report averaged results.
Since our focus is latency-efficient parallel TTS, we use wall-clock latency as the primary compute budget and compare methods by their accuracy--latency Pareto frontier.
For CPT, the reported latency is end-to-end and includes all operations executed by the method, including parallel generation, search-information extraction, semantic deduplication, input-context broadcast, and final answer aggregation. Prompt templates are provided in Appendix~\ref{app:prompts}.

\paragraph{Implementation Details.}
All experiments are executed in parallel on a cluster with 512 NVIDIA H800 GPUs, where each individual run is allocated to 8 GPUs.
To ensure reproducibility and fair comparison, we strictly follow the official recommended sampling hyperparameters for the Qwen3-Thinking series: temperature \(T=0.6\), top-\(p=0.95\), top-\(k=20\), with a maximum generation length of \(38\)k tokens.
For CPT, the same policy model is used for both reasoning generation and search-information extraction, without external reward models, verifiers, or task-specific training.
Unless otherwise specified, CPT uses a search-step length of \(2048\) tokens and samples at most \(M=512\) pool entries for each broadcast.
For semantic de-duplication, we encode extracted information units with \texttt{all-MiniLM-L6-v2}\footnote{\url{https://huggingface.co/sentence-transformers/all-MiniLM-L6-v2}} and use a cosine-similarity threshold of \(\tau_{\mathrm{dup}}=0.75\).
For adaptive broadcast scheduling, we compute the relative new-information gain over windows of \(W=3\) search steps, with \(\tau_{\mathrm{start}}=0.4\) and \(\tau_{\mathrm{stop}}=0.1\).
All CPT hyperparameters are kept fixed across benchmarks, models, and rollout budgets unless explicitly ablated.

\begin{table*}[t]
\centering
\small
\setlength{\tabcolsep}{3.0pt}
\resizebox{0.85\textwidth}{!}{
\begin{tabular}{llccccccccc}
\toprule
\multirow{2}{*}{\textbf{Benchmark}} 
& \multirow{2}{*}{\textbf{Metrics}} 
& \multicolumn{4}{c}{\textbf{Stop Threshold \(\tau_{\mathrm{stop}}\)}} 
& \multicolumn{5}{c}{\textbf{Start Threshold \(\tau_{\mathrm{start}}\)}} \\
\cmidrule(lr){3-6}\cmidrule(lr){7-11}
& & \textbf{0} & \textbf{0.10} & \textbf{0.20} & \textbf{0.30}
& \textbf{0.30} & \textbf{0.40} & \textbf{0.50} & \textbf{0.60} & \textbf{1.00} \\
\midrule
\multirow{3}{*}{HMMT24}
& Pass@1 (\%)              & 57.86 & 58.35 & 56.60 & 56.75 & 53.94 & 58.35 & 58.78 & 56.74 & 53.05 \\
& Tokens (\(\times 10^6\)) & 1.52  & 1.57  & 1.51  & 1.54  & 1.62  & 1.57  & 1.48  & 1.45  & 1.43  \\
& Latency (s)              & 482   & 457   & 413   & 395   & 436   & 457   & 453   & 447   & 449   \\
\midrule
\multirow{3}{*}{HMMT25}
& Pass@1 (\%)              & 66.94 & 66.79 & 65.33 & 64.58 & 59.67 & 66.79 & 65.28 & 66.22 & 63.37 \\
& Tokens (\(\times 10^6\)) & 1.43  & 1.47  & 1.42  & 1.42  & 1.50  & 1.47  & 1.38  & 1.34  & 1.29  \\
& Latency (s)              & 410   & 403   & 380   & 354   & 380   & 403   & 400   & 395   & 393   \\
\bottomrule
\end{tabular}
}
\caption{
Sensitivity analysis of adaptive broadcast scheduling thresholds using \textsc{Qwen3-4B-Thinking-2507} with 64 parallel samples.
For the stop-threshold sweep, \(\tau_{\mathrm{start}}=0.4\) is fixed; for the start-threshold sweep, \(\tau_{\mathrm{stop}}=0.1\) is fixed.
}
\label{tab:ablation_adaptive_thresholds}
\end{table*}

\subsection{Main Results}
\label{sec:main_results}

The main results across mathematical reasoning benchmarks and model scales are summarized in Figure~\ref{fig:main_results}.
CPT consistently establishes a stronger accuracy--latency Pareto frontier than the compared baselines across different rollout budgets, achieving higher accuracy under comparable wall-clock latency.
This advantage is observed on both \textsc{Qwen3-4B-Thinking-2507} and \textsc{Qwen3-30B-A3B-Thinking-2507}, and across AIME and HMMT benchmarks, indicating that search-time information sharing remains effective across model scales and problem distributions.
The simultaneous gains in Pass@1 and MV@K indicate that CPT improves individual branch quality and the overall effectiveness of parallel exploration through information sharing, rather than simply increasing final-answer consistency by prematurely aligning branch trajectories.
Overall, our results suggest that reducing redundant rediscovery through search-time information sharing is an effective mechanism for latency-efficient parallel TTS.

\begin{table*}[t]
\centering
\small
\setlength{\tabcolsep}{4.0pt}
\resizebox{0.9\textwidth}{!}{
\begin{tabular}{llccccc}
\toprule
\multirow{2}{*}{\textbf{Benchmark}} 
& \multirow{2}{*}{\textbf{Model}}
& \multicolumn{3}{c}{\textbf{Latency}} 
& \multicolumn{2}{c}{\textbf{FLOPs}} \\
\cmidrule(lr){3-5} \cmidrule(lr){6-7}
& & \textbf{Sampling} & \textbf{Info Extract} & \textbf{Dedup. \& Filter} 
  & \textbf{Sampling} & \textbf{Info Extract} \\
\midrule
\multirow{2}{*}{AIME25} 
& \textsc{Qwen3-4B}      & 307.90 & 35.65 & 18.16 & 112.43 & 1.66 \\
& \textsc{Qwen3-30B-A3B} & 420.37 & 61.59 & 29.44 & 81.99  & 1.70 \\
\midrule
\multirow{2}{*}{AIME26} 
& \textsc{Qwen3-4B}      & 291.00 & 39.79 & 21.92 & 107.11 & 1.50 \\
& \textsc{Qwen3-30B-A3B} & 415.13 & 50.42 & 24.86 & 82.02  & 1.80 \\
\midrule
\multirow{2}{*}{HMMT24} 
& \textsc{Qwen3-4B}      & 368.43 & 56.75 & 26.95 & 138.52 & 1.62 \\
& \textsc{Qwen3-30B-A3B} & 607.90 & 83.55 & 34.46 & 141.12 & 2.64 \\
\midrule
\multirow{2}{*}{HMMT25} 
& \textsc{Qwen3-4B}      & 331.72 & 44.42 & 22.39 & 120.02 & 1.50 \\
& \textsc{Qwen3-30B-A3B} & 582.40 & 75.03 & 32.27 & 132.76 & 2.35 \\
\bottomrule
\end{tabular}
}
\caption{
Component-wise latency and FLOPs of CPT with 64 parallel samples.
Latency is reported in seconds and FLOPs are reported in PFLOPs per question.
}
\label{tab:cpt_component_cost}
\end{table*}

\begin{table}[t]
\centering
\small
\setlength{\tabcolsep}{3.0pt}
\begin{tabular}{lcccccc}
\toprule
\multirow{2}{*}{\textbf{Metrics}} 
& \multicolumn{6}{c}{\textbf{Broadcast Size \(M\)}} \\ 
\cmidrule{2-7}
& \textbf{32} & \textbf{64} & \textbf{128} & \textbf{256} & \textbf{512} & \textbf{1024} \\ 
\midrule
Pass@1 (\%)              & 61.65 & 61.93 & 63.68 & 64.03 & 66.79 & 65.32 \\
Tokens (\(\times 10^6\)) & 1.53  & 1.48  & 1.43  & 1.42  & 1.42  & 1.42  \\
Latency (s)              & 429   & 411   & 408   & 407   & 402   & 385   \\
\bottomrule
\end{tabular}
\caption{
Effect of broadcast size \(M\) on HMMT25 using \textsc{Qwen3-4B-Thinking-2507} with 64 parallel samples. We vary the maximum number of pool entries inserted into the shared-information section at each broadcast.
}
\label{tab:ablation_broadcast_size}
\end{table}

\begin{table}[t]
\centering
\small
\setlength{\tabcolsep}{5.0pt}
\begin{tabular}{lccccc}
\toprule
\multirow{2}{*}{\textbf{Metrics}} 
& \multicolumn{5}{c}{\textbf{Similarity Threshold \(\tau_{\mathrm{dup}}\)}} \\ 
\cmidrule{2-6}
& \textbf{0.65} & \textbf{0.70} & \textbf{0.75} & \textbf{0.80} & \textbf{0.85} \\ 
\midrule
Pass@1 (\%)              & 62.60 & 64.57 & 66.79 & 65.92 & 66.66 \\
Tokens (\(\times 10^6\)) & 1.51  & 1.48  & 1.46  & 1.46  & 1.47  \\
Latency (s)              & 383   & 385   & 402   & 400   & 405   \\
\bottomrule
\end{tabular}
\caption{
Effect of similarity threshold \(\tau_{\mathrm{dup}}\) on HMMT25 using \textsc{Qwen3-4B-Thinking-2507} with 64 parallel samples.
We vary the semantic de-duplication threshold used to filter information units before adding them to the shared pool.
}
\label{tab:ablation_similarity_threshold}
\end{table}


\subsection{Analysis}
\label{sec:analysis}

\paragraph{Information Redundancy Reduction.}
Beyond the accuracy--latency frontier, we analyze how CPT changes the information content of parallel search.
As shown in Figure~\ref{fig:gain_duplication}, independent parallel sampling yields more duplicate information units per search step as search proceeds, while the number of newly discovered information units decreases.
This indicates that later exploration increasingly spends computation rediscovering information already found by other branches.
In contrast, CPT reduces per-step duplicate discovery while maintaining a comparable level of new-information discovery.
This suggests that the gains of CPT do not merely come from forcing branches to agree earlier, but from reducing the residual search burden through cross-branch information sharing.

\paragraph{Effect of Adaptive Broadcast Scheduling.}
Table~\ref{tab:ablation_adaptive_thresholds} analyzes the effect of the adaptive start and stop thresholds.
When fixing \(\tau_{\mathrm{start}}=0.4\), setting \(\tau_{\mathrm{stop}}=0\) disables early stopping of synchronization, while larger \(\tau_{\mathrm{stop}}\) values stop broadcasting earlier.
The results show that moderate early stopping generally provides a better accuracy--latency trade-off than either never stopping or stopping too aggressively, suggesting that continued synchronization becomes less useful once marginal new-information gain is low.
When fixing \(\tau_{\mathrm{stop}}=0.1\), setting \(\tau_{\mathrm{start}}=1.0\) starts broadcasting immediately, while smaller values allow branches to explore independently before sharing begins.
Immediate broadcasting does not yield the best accuracy, indicating that overly early sharing can prematurely align exploration directions.
Overall, these results support our adaptive scheduling design: CPT should first preserve independent exploration, then activate sharing when redundant discovery increases, and finally stop synchronization once additional broadcasts provide limited benefit.

\paragraph{Hyperparameter Sensitivity.}
We analyze the sensitivity of CPT to two key hyperparameters: the broadcast size \(M\) and the semantic de-duplication threshold \(\tau_{\mathrm{dup}}\).
As shown in Table~\ref{tab:ablation_broadcast_size}, increasing \(M\) generally improves Pass@1 and reduces both generated tokens and latency, indicating that providing more shared information can help branches avoid redundant exploration and converge more efficiently.
However, the gain saturates when \(M\) becomes large, and \(M=512\) achieves the best accuracy in our setting.
Table~\ref{tab:ablation_similarity_threshold} shows that CPT is relatively stable across a range of de-duplication thresholds, with \(\tau_{\mathrm{dup}}=0.75\) giving the best Pass@1.
A threshold that is too low may filter out useful but related information, while a threshold that is too high may admit redundant entries into the shared pool.
We therefore use \(M=512\) and \(\tau_{\mathrm{dup}}=0.75\) as the default configuration in our main experiments.

\paragraph{Component-wise Runtime Cost.}
Table~\ref{tab:cpt_component_cost} reports the component-wise latency and FLOPs of CPT at \(N=64\).
Parallel sampling remains the dominant source of wall-clock latency, while information extraction and semantic de-duplication/filtering add moderate overhead.
The FLOPs of information extraction are small compared with sampling, suggesting that the extraction itself is not the main computational bottleneck.
Instead, the main FLOPs trade-off arises from prompt-level broadcasting, which may require re-prefilling updated contexts in standard decoding implementations.
This motivates future cache-aware or attention-level sharing mechanisms to further reduce the computational overhead of CPT.

\paragraph{Generated Tokens and FLOPs.}
Although our main comparison uses wall-clock latency as the primary compute budget, we also examine generated tokens and FLOPs as auxiliary cost metrics in Appendix~\ref{app:flops_tokens_analysis}.
The accuracy--tokens results (Figure~\ref{fig:main_results_tokens}) show that CPT achieves a stronger token-efficiency frontier, indicating that search-time information sharing reduces unnecessary decoding spent on rediscovering information already found by other branches.
The FLOPs comparison (Figure~\ref{fig:main_results_flops}) reveals a different trade-off: as CPT implements sharing through input-context updates, standard decoding implementations may require re-prefilling the updated context, introducing additional computation.
Thus, CPT improves latency-efficient parallel exploration and reduces redundant generated tokens, while more cache-aware or attention-level sharing mechanisms are needed to further improve FLOPs efficiency.

\section{Conclusions}

In this work, we identify and validate an information-isolation bottleneck in parallel test-time scaling through preliminary studies: branch-private discoveries remain unshared during search, leading to redundant rediscovery and slower convergence toward correct answers. To bridge this gap, we introduce Collaborative Parallel Thinking (CPT), a training-free inference framework that enables search-time information sharing across parallel branches. CPT builds a shared information pool by extracting compact intermediate information from ongoing branches and merging overlapping entries. 
Pool entries are then broadcast through the input context, turning branch-private discoveries into reusable global decision information. Empirical evaluations on challenging mathematical reasoning benchmarks demonstrate that CPT establishes a stronger accuracy--latency Pareto frontier than strong baselines across rollout budgets and model scales. 
Overall, our findings highlight that the efficiency of parallel TTS depends not only on launching more reasoning branches, but also on enabling branches to share and reuse decision-relevant information during search.
\section*{Limitations}
Collaborative Parallel Thinking (CPT) uses a fixed-token synchronous protocol to make branch discoveries reusable during parallel search, balancing timely information reuse against synchronization and context-update overhead. Since the shared information pool is updated only after each search step, discoveries made within a step cannot influence other branches until the next step, so duplicate discoveries may still occur before the next broadcast. In addition, broadcasting the updated pool through the input context is not free: in standard decoding implementations, it can require re-prefilling the updated context, introducing additional FLOPs and latency overhead. Future work could reduce this cost by designing attention-level or cache-aware mechanisms that reuse prior computation, support lower-cost context updates, or enable finer-grained sharing without repeatedly recomputing the full prefix.
\bibliography{anthology,custom}

@article{baseline-deep-c,
  author       = {Yichao Fu and
                  Xuewei Wang and
                  Yuandong Tian and
                  Jiawei Zhao},
  title        = {Deep Think with Confidence},
  journal      = {CoRR},
  volume       = {abs/2508.15260},
  year         = {2025},
  url          = {https://doi.org/10.48550/arXiv.2508.15260},
  doi          = {10.48550/ARXIV.2508.15260},
  eprinttype   = {arXiv},
  eprint       = {2508.15260},
  bibsource    = {dblp computer science bibliography, https://dblp.org}
}

@inproceedings{self-consistency,
  author       = {Xuezhi Wang and
                  Jason Wei and
                  Dale Schuurmans and
                  Quoc V. Le and
                  Ed H. Chi and
                  Sharan Narang and
                  Aakanksha Chowdhery and
                  Denny Zhou},
  title        = {Self-Consistency Improves Chain of Thought Reasoning in Language Models},
  booktitle    = {The Eleventh International Conference on Learning Representations,
                  {ICLR} 2023, Kigali, Rwanda, May 1-5, 2023},
  publisher    = {OpenReview.net},
  year         = {2023},
  url          = {https://openreview.net/forum?id=1PL1NIMMrw},
  bibsource    = {dblp computer science bibliography, https://dblp.org}
}

@article{llm-qwen3,
  author       = {An Yang and
                  Anfeng Li and
                  Baosong Yang and
                  Beichen Zhang and
                  Binyuan Hui and
                  Bo Zheng and
                  Bowen Yu and
                  Chang Gao and
                  Chengen Huang and
                  Chenxu Lv and
                  Chujie Zheng and
                  Dayiheng Liu and
                  Fan Zhou and
                  Fei Huang and
                  Feng Hu and
                  Hao Ge and
                  Haoran Wei and
                  Huan Lin and
                  Jialong Tang and
                  Jian Yang and
                  Jianhong Tu and
                  Jianwei Zhang and
                  Jian Yang and
                  Jiaxi Yang and
                  Jingren Zhou and
                  Junyang Lin and
                  Kai Dang and
                  Keqin Bao and
                  Kexin Yang and
                  Le Yu and
                  Lianghao Deng and
                  Mei Li and
                  Mingfeng Xue and
                  Mingze Li and
                  Pei Zhang and
                  Peng Wang and
                  Qin Zhu and
                  Rui Men and
                  Ruize Gao and
                  Shixuan Liu and
                  Shuang Luo and
                  Tianhao Li and
                  Tianyi Tang and
                  Wenbiao Yin and
                  Xingzhang Ren and
                  Xinyu Wang and
                  Xinyu Zhang and
                  Xuancheng Ren and
                  Yang Fan and
                  Yang Su and
                  Yichang Zhang and
                  Yinger Zhang and
                  Yu Wan and
                  Yuqiong Liu and
                  Zekun Wang and
                  Zeyu Cui and
                  Zhenru Zhang and
                  Zhipeng Zhou and
                  Zihan Qiu},
  title        = {Qwen3 Technical Report},
  journal      = {CoRR},
  volume       = {abs/2505.09388},
  year         = {2025},
  url          = {https://doi.org/10.48550/arXiv.2505.09388},
  doi          = {10.48550/ARXIV.2505.09388},
  eprinttype    = {arXiv},
  eprint       = {2505.09388},
  bibsource    = {dblp computer science bibliography, https://dblp.org}
}

@inproceedings{ESC,
  title={Escape Sky-high Cost: Early-stopping Self-Consistency for Multi-step Reasoning},
  author={Li, Yiwei and Yuan, Peiwen and Feng, Shaoxiong and Pan, Boyuan and Wang, Xinglin and Sun, Bin and Wang, Heda and Li, Kan},
  booktitle={The Twelfth International Conference on Learning Representations},
  year={2024}
}

@article{wei2022chain,
  title={Chain-of-thought prompting elicits reasoning in large language models},
  author={Wei, Jason and Wang, Xuezhi and Schuurmans, Dale and Bosma, Maarten and Xia, Fei and Chi, Ed and Le, Quoc V and Zhou, Denny and others},
  journal={Advances in neural information processing systems},
  volume={35},
  pages={24824--24837},
  year={2022}
}

@article{snell2024scaling,
  title={Scaling llm test-time compute optimally can be more effective than scaling model parameters},
  author={Snell, Charlie and Lee, Jaehoon and Xu, Kelvin and Kumar, Aviral},
  journal={arXiv preprint arXiv:2408.03314},
  year={2024}
}

@inproceedings{wu2025inference,
  title={Inference scaling laws: An empirical analysis of compute-optimal inference for LLM problem-solving},
  author={Wu, Yangzhen and Sun, Zhiqing and Li, Shanda and Welleck, Sean and Yang, Yiming},
  booktitle={The Thirteenth International Conference on Learning Representations},
  year={2025}
}

@article{liu2025can,
  title={Can 1B LLM Surpass 405B LLM? Rethinking Compute-Optimal Test-Time Scaling},
  author={Liu, Runze and Gao, Junqi and Zhao, Jian and Zhang, Kaiyan and Li, Xiu and Qi, Biqing and Ouyang, Wanli and Zhou, Bowen},
  journal={arXiv preprint arXiv:2502.06703},
  year={2025}
}

@article{brown2024large,
  title={Large language monkeys: Scaling inference compute with repeated sampling},
  author={Brown, Bradley and Juravsky, Jordan and Ehrlich, Ryan and Clark, Ronald and Le, Quoc V and R{\'e}, Christopher and Mirhoseini, Azalia},
  journal={arXiv preprint arXiv:2407.21787},
  year={2024}
}

@inproceedings{lightman2023let,
  title={Let's verify step by step},
  author={Lightman, Hunter and Kosaraju, Vineet and Burda, Yuri and Edwards, Harrison and Baker, Bowen and Lee, Teddy and Leike, Jan and Schulman, John and Sutskever, Ilya and Cobbe, Karl},
  booktitle={The Twelfth International Conference on Learning Representations},
  year={2024}
}

@misc{o1,
  title  = {Learning to Reason with LLMs},
  author = {OpenAI},
  url    = {https://openai.com/index/learning-to-reason-with-llms/},
  year   = {2024}
}

@misc{AIME24,
  title  = {AIME 2024},
  author = {{AI-MO}},
  url    = {https://huggingface.co/datasets/AI-MO/aimo-validation-aime},
  year   = {2024}
}

@article{liu2023don,
  title={Don't throw away your value model! Generating more preferable text with Value-Guided Monte-Carlo Tree Search decoding},
  author={Liu, Jiacheng and Cohen, Andrew and Pasunuru, Ramakanth and Choi, Yejin and Hajishirzi, Hannaneh and Celikyilmaz, Asli},
  journal={arXiv preprint arXiv:2309.15028},
  year={2023}
}

@article{Kimi-k1.5,
  title   = {Kimi k1.5: Scaling Reinforcement Learning with LLMs},
  author  = {{Kimi Team} and Du, Angang and Gao, Bofei and Xing, Bowei and Jiang, Changjiu and Chen, Cheng and Li, Cheng and Xiao, Chenjun and Du, Chenzhuang and Liao, Chonghua and others},
  journal = {arXiv preprint arXiv:2501.12599},
  year    = {2025}
}

@article{DeepSeek-R1,
  title   = {DeepSeek-R1: Incentivizing Reasoning Capability in LLMs via Reinforcement Learning},
  author  = {{DeepSeek-AI} and Daya Guo and Dejian Yang and Haowei Zhang and Junxiao Song and Ruoyu Zhang and Runxin Xu and Qihao Zhu and Shirong Ma and Peiyi Wang and Xiao Bi and Xiaokang Zhang and Xingkai Yu and Yu Wu and Z. F. Wu and Zhibin Gou and Zhihong Shao and Zhuoshu Li and Ziyi Gao and Aixin Liu and Bing Xue and Bingxuan Wang and Bochao Wu and Bei Feng and Chengda Lu and Chenggang Zhao and Chengqi Deng and Chenyu Zhang and Chong Ruan and Damai Dai and Deli Chen and Dongjie Ji and Erhang Li and Fangyun Lin and Fucong Dai and Fuli Luo and Guangbo Hao and Guanting Chen and Guowei Li and H. Zhang and Han Bao and Hanwei Xu and Haocheng Wang and Honghui Ding and Huajian Xin and Huazuo Gao and Hui Qu and Hui Li and Jianzhong Guo and Jiashi Li and Jiawei Wang and Jingchang Chen and Jingyang Yuan and Junjie Qiu and Junlong Li and J. L. Cai and Jiaqi Ni and Jian Liang and Jin Chen and Kai Dong and Kai Hu and Kaige Gao and Kang Guan and Kexin Huang and Kuai Yu and Lean Wang and Lecong Zhang and Liang Zhao and Litong Wang and Liyue Zhang and Lei Xu and Leyi Xia and Mingchuan Zhang and Minghua Zhang and Minghui Tang and Meng Li and Miaojun Wang and Mingming Li and Ning Tian and Panpan Huang and Peng Zhang and Qiancheng Wang and Qinyu Chen and Qiushi Du and Ruiqi Ge and Ruisong Zhang and Ruizhe Pan and Runji Wang and R. J. Chen and R. L. Jin and Ruyi Chen and Shanghao Lu and Shangyan Zhou and Shanhuang Chen and Shengfeng Ye and Shiyu Wang and Shuiping Yu and Shunfeng Zhou and Shuting Pan and S. S. Li and Shuang Zhou and Shaoqing Wu and Shengfeng Ye and Tao Yun and Tian Pei and Tianyu Sun and T. Wang and Wangding Zeng and Wanjia Zhao and Wen Liu and Wenfeng Liang and Wenjun Gao and Wenqin Yu and Wentao Zhang and W. L. Xiao and Wei An and Xiaodong Liu and Xiaohan Wang and Xiaokang Chen and Xiaotao Nie and Xin Cheng and Xin Liu and Xin Xie and Xingchao Liu and Xinyu Yang and Xinyuan Li and Xuecheng Su and Xuheng Lin and X. Q. Li and Xiangyue Jin and Xiaojin Shen and Xiaosha Chen and Xiaowen Sun and Xiaoxiang Wang and Xinnan Song and Xinyi Zhou and Xianzu Wang and Xinxia Shan and Y. K. Li and Y. Q. Wang and Y. X. Wei and Yang Zhang and Yanhong Xu and Yao Li and Yao Zhao and Yaofeng Sun and Yaohui Wang and Yi Yu and Yichao Zhang and Yifan Shi and Yiliang Xiong and Ying He and Yishi Piao and Yisong Wang and Yixuan Tan and Yiyang Ma and Yiyuan Liu and Yongqiang Guo and Yuan Ou and Yuduan Wang and Yue Gong and Yuheng Zou and Yujia He and Yunfan Xiong and Yuxiang Luo and Yuxiang You and Yuxuan Liu and Yuyang Zhou and Y. X. Zhu and Yanhong Xu and Yanping Huang and Yaohui Li and Yi Zheng and Yuchen Zhu and Yunxian Ma and Ying Tang and Yukun Zha and Yuting Yan and Z. Z. Ren and Zehui Ren and Zhangli Sha and Zhe Fu and Zhean Xu and Zhenda Xie and Zhengyan Zhang and Zhewen Hao and Zhicheng Ma and Zhigang Yan and Zhiyu Wu and Zihui Gu and Zijia Zhu and Zijun Liu and Zilin Li and Ziwei Xie and Ziyang Song and Zizheng Pan and Zhen Huang and Zhipeng Xu and Zhongyu Zhang and Zhen Zhang},
  journal = {arXiv preprint arXiv:2501.12948},
  year    = {2025}
}

@inproceedings{wan2024alphazero,
  title     = {{A}lpha{Z}ero-Like Tree-Search can Guide Large Language Model Decoding and Training},
  author    = {Wan, Ziyu and Feng, Xidong and Wen, Muning and Mcaleer, Stephen Marcus and Wen, Ying and Zhang, Weinan and Wang, Jun},
  booktitle = {International Conference on Machine Learning (ICML)},
  pages     = {49890--49920},
  volume    = {235},
  year      = {2024}
}

@inproceedings{Self-Refine,
  title     = {Self-Refine: Iterative Refinement with Self-Feedback},
  author    = {Madaan, Aman and Tandon, Niket and Gupta, Prakhar and Hallinan, Skyler and Gao, Luyu and Wiegreffe, Sarah and Alon, Uri and Dziri, Nouha and Prabhumoye, Shrimai and Yang, Yiming and Gupta, Shashank and Majumder, Bodhisattwa Prasad and Hermann, Katherine and Welleck, Sean and Yazdanbakhsh, Amir and Clark, Peter},
  booktitle = {Advances in Neural Information Processing Systems (NeurIPS)},
  pages     = {46534--46594},
  volume    = {36},
  year      = {2023}
}

@article{yao2023tree,
  title={Tree of thoughts: Deliberate problem solving with large language models},
  author={Yao, Shunyu and Yu, Dian and Zhao, Jeffrey and Shafran, Izhak and Griffiths, Tom and Cao, Yuan and Narasimhan, Karthik},
  journal={Advances in neural information processing systems},
  volume={36},
  pages={11809--11822},
  year={2023}
}

@article{zhang2025survey,
  title={A Survey on Test-Time Scaling in Large Language Models: What, How, Where, and How Well?},
  author={Zhang, Qiyuan and Lyu, Fuyuan and Sun, Zexu and Wang, Lei and Zhang, Weixu and Hua, Wenyue and Wu, Haolun and Guo, Zhihan and Wang, Yufei and Muennighoff, Niklas and others},
  journal={arXiv preprint arXiv:2503.24235},
  year={2025}
}

@misc{balunovic_srimatharena_2025,
  title = {MathArena: Evaluating LLMs on Uncontaminated Math Competitions},
  author = {Mislav Balunović and Jasper Dekoninck and Ivo Petrov and Nikola Jovanović and Martin Vechev},
  copyright = {MIT},
  url = {https://matharena.ai/},
  publisher = {SRI Lab, ETH Zurich},
  month = feb,
  year = {2025},
}

@article{shinn2023reflexion,
  title={Reflexion: Language agents with verbal reinforcement learning},
  author={Shinn, Noah and Cassano, Federico and Gopinath, Ashwin and Narasimhan, Karthik and Yao, Shunyu},
  journal={Advances in Neural Information Processing Systems},
  volume={36},
  pages={8634--8652},
  year={2023}
}

@misc{nemo-rl,
title = {NeMo RL: A Scalable and Efficient Post-Training Library},
howpublished = {\url{https://github.com/NVIDIA-NeMo/RL}},
year = {2025},
note = {GitHub repository},
}

@article{Wang2025EveryRC,
  title={Every Rollout Counts: Optimal Resource Allocation for Efficient Test-Time Scaling},
  author={Xinglin Wang and Yiwei Li and Shaoxiong Feng and Peiwen Yuan and Yueqi Zhang and Jiayi Shi and Chuyi Tan and Boyuan Pan and Yao Hu and Kan Li},
  journal={ArXiv},
  year={2025},
  volume={abs/2506.15707},
  url={https://api.semanticscholar.org/CorpusID:279464512}
}

@article{zheng2026parallel,
  title={Parallel-Probe: Towards Efficient Parallel Thinking via 2D Probing},
  author={Zheng, Tong and Huang, Chengsong and Dai, Runpeng and He, Yun and Liu, Rui and Ni, Xin and Bao, Huiwen and Wang, Kaishen and Zhu, Hongtu and Huang, Jiaxin and others},
  journal={arXiv preprint arXiv:2602.03845},
  year={2026}
}

@article{luo2025learning,
  title={Learning from Peers in Reasoning Models},
  author={Luo, Tongxu and Du, Wenyu and Bi, Jiaxi and Chung, Stephen and Tang, Zhengyang and Yang, Hao and Zhang, Min and Wang, Benyou},
  journal={arXiv preprint arXiv:2505.07787},
  year={2025}
}

@article{wang2025faster,
  title={Faster and better llms via latency-aware test-time scaling},
  author={Wang, Zili and Zhang, Tianyu and Bai, Haoli and Hou, Lu and Yu, Xianzhi and Liu, Wulong and Xiang, Shiming and Zhu, Lei},
  journal={arXiv preprint arXiv:2505.19634},
  year={2025}
}

@article{wang2026not,
  title={Do Not Waste Your Rollouts: Recycling Search Experience for Efficient Test-Time Scaling},
  author={Wang, Xinglin and Shi, Jiayi and Feng, Shaoxiong and Yuan, Peiwen and Li, Yiwei and Zhang, Yueqi and Tan, Chuyi and Zhang, Ji and Pan, Boyuan and Hu, Yao and others},
  journal={arXiv preprint arXiv:2601.21684},
  year={2026}
}

@inproceedings{liang2024encouraging,
  title={Encouraging divergent thinking in large language models through multi-agent debate},
  author={Liang, Tian and He, Zhiwei and Jiao, Wenxiang and Wang, Xing and Wang, Yan and Wang, Rui and Yang, Yujiu and Shi, Shuming and Tu, Zhaopeng},
  booktitle={Proceedings of the 2024 conference on empirical methods in natural language processing},
  pages={17889--17904},
  year={2024}
}

@inproceedings{chen2024reconcile,
  title={Reconcile: Round-table conference improves reasoning via consensus among diverse llms},
  author={Chen, Justin and Saha, Swarnadeep and Bansal, Mohit},
  booktitle={Proceedings of the 62nd Annual Meeting of the Association for Computational Linguistics (Volume 1: Long Papers)},
  pages={7066--7085},
  year={2024}
}

@article{estornell2024multi,
  title={Multi-llm debate: Framework, principals, and interventions},
  author={Estornell, Andrew and Liu, Yang},
  journal={Advances in Neural Information Processing Systems},
  volume={37},
  pages={28938--28964},
  year={2024}
}

@inproceedings{liu2025breaking,
  title={Breaking mental set to improve reasoning through diverse multi-agent debate},
  author={Liu, Yexiang and Cao, Jie and Li, Zekun and He, Ran and Tan, Tieniu},
  booktitle={The Thirteenth International Conference on Learning Representations},
  year={2025}
}

@inproceedings{wang2024rethinking,
  title={Rethinking the bounds of llm reasoning: Are multi-agent discussions the key?},
  author={Wang, Qineng and Wang, Zihao and Su, Ying and Tong, Hanghang and Song, Yangqiu},
  booktitle={Proceedings of the 62nd Annual Meeting of the Association for Computational Linguistics (Volume 1: Long Papers)},
  pages={6106--6131},
  year={2024}
}

@article{choi2026debate,
  title={Debate or vote: Which yields better decisions in multi-agent large language models?},
  author={Choi, Hyeong Kyu and Zhu, Jerry and Li, Sharon},
  journal={Advances in Neural Information Processing Systems},
  volume={38},
  pages={101732--101764},
  year={2026}
}

@inproceedings{li2024improving,
  title={Improving multi-agent debate with sparse communication topology},
  author={Li, Yunxuan and Du, Yibing and Zhang, Jiageng and Hou, Le and Grabowski, Peter and Li, Yeqing and Ie, Eugene},
  booktitle={Findings of the Association for Computational Linguistics: EMNLP 2024},
  pages={7281--7294},
  year={2024}
}

@inproceedings{zeng2025s2,
  title={S2-mad: Breaking the token barrier to enhance multi-agent debate efficiency},
  author={Zeng, Yuting and Huang, Weizhe and Jiang, Lei and Liu, Tongxuan and Jin, Xitai and Tiana, Chen Tianying and Li, Jing and Xu, Xiaohua},
  booktitle={Proceedings of the 2025 Conference of the Nations of the Americas Chapter of the Association for Computational Linguistics: Human Language Technologies (Volume 1: Long Papers)},
  pages={9393--9408},
  year={2025}
}

@article{liu2024groupdebate,
  title={Groupdebate: Enhancing the efficiency of multi-agent debate using group discussion},
  author={Liu, Tongxuan and Wang, Xingyu and Huang, Weizhe and Xu, Wenjiang and Zeng, Yuting and Jiang, Lei and Yang, Hailong and Li, Jing},
  journal={arXiv preprint arXiv:2409.14051},
  year={2024}
}

@article{eo2025debate,
  title={Debate only when necessary: Adaptive multiagent collaboration for efficient llm reasoning},
  author={Eo, Sugyeong and Moon, Hyeonseok and Zi, Evelyn Hayoon and Park, Chanjun and Lim, Heuiseok},
  journal={arXiv preprint arXiv:2504.05047},
  year={2025}
}

@inproceedings{wang2025mixture,
  title={Mixture-of-agents enhances large language model capabilities},
  author={Wang, Junlin and Wang, Jue and Athiwaratkun, Ben and Zhang, Ce and Zou, James Y},
  booktitle={International Conference on Learning Representations},
  volume={2025},
  pages={33944--33963},
  year={2025}
}

@inproceedings{qian2025scaling,
  title={Scaling large language model-based multi-agent collaboration},
  author={Qian, Chen and Xie, Zihao and Wang, Yifei and Liu, Wei and Zhu, Kunlun and Xia, Hanchen and Dang, Yufan and Du, Zhuoyun and Chen, Weize and Yang, Cheng and others},
  booktitle={International Conference on Learning Representations},
  volume={2025},
  pages={41488--41505},
  year={2025}
}

@article{wang2025survey,
  title={A survey on parallel reasoning},
  author={Wang, Ziqi and Niu, Boye and Gao, Zipeng and Zheng, Zhi and Xu, Tong and Meng, Linghui and Li, Zhongli and Liu, Jing and Chen, Yilong and Zhu, Chen and others},
  journal={arXiv preprint arXiv:2510.12164},
  year={2025}
}

@article{AIME26,
      title={Beyond Benchmarks: MathArena as an Evaluation Platform for Mathematics with LLMs}, 
      author={Jasper Dekoninck and Nikola Jovanović and Tim Gehrunger and Kári Rögnvaldsson and Ivo Petrov and Chenhao Sun and Martin Vechev},
      year={2026},
      eprint={2605.00674},
      archivePrefix={arXiv},
      primaryClass={cs.CL},
      url={https://arxiv.org/abs/2605.00674}, 
}

@article{zheng2026llms,
  title={LLMs Improving LLMs: Agentic Discovery for Test-Time Scaling},
  author={Zheng, Tong and Liu, Haolin and Huang, Chengsong and Bao, Huiwen and Zhang, Sheng and Liu, Rui and Dai, Runpeng and Chen, Ruibo and Liu, Chenxi and Xiong, Tianyi and others},
  journal={arXiv preprint arXiv:2605.08083},
  year={2026}
}

@article{zheng2025parallel,
  title={Parallel-r1: Towards parallel thinking via reinforcement learning},
  author={Zheng, Tong and Zhang, Hongming and Yu, Wenhao and Wang, Xiaoyang and Dai, Runpeng and Liu, Rui and Bao, Huiwen and Huang, Chengsong and Huang, Heng and Yu, Dong},
  journal={arXiv preprint arXiv:2509.07980},
  year={2025}
}
\clearpage
\appendix

\section{Collaborative Parallel Thinking Algorithm}
\label{app:cpt_algorithm}

Algorithm~\ref{alg:cpt} provides the full procedure of Collaborative Parallel Thinking (CPT).
For a problem \(x\), CPT maintains \(K\) private branch histories \(\{h_i\}_{i=1}^K\) and a query-level shared information pool \(\mathcal{P}\).
The procedure alternates between fixed-token parallel generation and pool update: at each search step, unfinished branches decode up to \(C\) tokens conditioned on the current broadcast set \(\mathcal{B}\), after which the same policy model extracts compact information from the newly generated segments.
Candidate information units are added to \(\mathcal{P}\) only when they are not semantically similar to existing entries, keeping the pool as a compact and deduplicated state of branch discoveries.

The broadcast schedule is controlled by the new-information record \(\mathcal{N}\).
CPT starts in the \textsc{Probe} mode, where branches explore independently while still writing extracted information into the pool.
Once the windowed relative information gain falls below \(\tau_{\mathrm{start}}\), CPT enters the \textsc{Broadcast} mode and samples pool entries to all branches through the shared-information section of the input context.
When the relative gain further falls below \(\tau_{\mathrm{stop}}\), CPT enters the \textsc{FreeRun} mode and lets the remaining branches decode until EOS or \(L_{\max}\) without further synchronization.
This implements bounded search-time sharing while preserving each branch's private reasoning history.

\begin{algorithm*}[!t]
\caption{Collaborative Parallel Thinking (CPT)}
\label{alg:cpt}
\small
\begin{algorithmic}[1]
\REQUIRE Problem \(x\), model \(\pi\), branch number \(K\), search-step length \(C\), maximum length \(L_{\max}\), broadcast size \(M\), deduplication threshold \(\tau_{\mathrm{dup}}\), window size \(W\), start threshold \(\tau_{\mathrm{start}}\), stop threshold \(\tau_{\mathrm{stop}}\), numerical constant \(\epsilon\).
\ENSURE Final reasoning branches \(\{h_i\}_{i=1}^{K}\).
\STATE Initialize branch histories \(h_i\gets\emptyset\) for all \(i\in[K]\); shared pool \(\mathcal{P}\gets\emptyset\); mode \(\gets \textsc{Probe}\); new-information record \(\mathcal{N}\gets[]\); reference gain \(g_{\mathrm{ref}}\gets\textsc{None}\).
\WHILE{some branch is unfinished, i.e., it has reached neither EOS nor \(L_{\max}\)}

    \STATE \textcolor{gray}{\small \textbf{// Component 1: Collaborative Parallel Search (Sec.~\ref{sec:cpt_search})}}
    \IF{mode \(=\textsc{Broadcast}\)}
        \STATE \(\mathcal{B}\gets \mathcal{P}\) if \(|\mathcal{P}|\le M\); otherwise sample \(M\) entries from \(\mathcal{P}\).
    \ELSE
        \STATE \(\mathcal{B}\gets\emptyset\).
    \ENDIF
    \FOR{\textbf{each unfinished branch} \(i\in[K]\) \textbf{in parallel}}
        \STATE \(\Delta h_i \sim \pi(\cdot \mid \mathsf{Prompt}(x,h_i,\mathcal{B}))\) for at most \(C\) tokens, or until EOS or \(L_{\max}\) is reached.
        \STATE \(h_i\gets h_i\oplus \Delta h_i\).
    \ENDFOR
    \IF{all branches have reached EOS or \(L_{\max}\)}
        \STATE \textbf{break}
    \ENDIF

    \STATE \textcolor{gray}{\small \textbf{// Component 2: Shared Information Pooling (Sec.~\ref{sec:pool_broadcast})}}
    \STATE \(n\gets 0\).
    \FOR{\textbf{each branch} \(i\) \textbf{with newly generated segment} \(\Delta h_i\)}
        \STATE \(\mathcal{Z}_i\gets \mathsf{Extract}_{\pi}(x,h_i,\Delta h_i)\).
        \FOR{\textbf{each} \(z\in\mathcal{Z}_i\)}
            \IF{\(\mathcal{P}=\emptyset\)}
                \STATE \(\mathcal{P}\gets\mathcal{P}\cup\{z\}\); \quad \(n\gets n+1\).
            \ELSE
                \STATE \(s^\star\gets\max_{p\in\mathcal{P}}\cos(\phi(z),\phi(p))\).
                \IF{\(s^\star<\tau_{\mathrm{dup}}\)}
                    \STATE \(\mathcal{P}\gets\mathcal{P}\cup\{z\}\); \quad \(n\gets n+1\).
                \ENDIF
            \ENDIF
        \ENDFOR
    \ENDFOR
    \STATE Append \(n\) to \(\mathcal{N}\).

    \STATE \textcolor{gray}{\small \textbf{// Component 3: Adaptive Broadcast Scheduling (Sec.~\ref{sec:broadcast_schedule})}}
    \IF{\(|\mathcal{N}| \equiv 0 \pmod{W}\)}
        \STATE \(j\gets |\mathcal{N}|/W\); \quad \(g_j\gets \frac{1}{W}\sum_{s=(j-1)W+1}^{jW}\mathcal{N}[s]\).
        \IF{\(g_{\mathrm{ref}}=\textsc{None}\)}
            \STATE \(g_{\mathrm{ref}}\gets g_j\).
        \ELSE
            \STATE \(r_j\gets g_j/(g_{\mathrm{ref}}+\epsilon)\).
            \IF{mode \(=\textsc{Probe}\) \textbf{and} \(r_j<\tau_{\mathrm{start}}\)}
                \STATE mode \(\gets\textsc{Broadcast}\).
            \ELSIF{mode \(=\textsc{Broadcast}\) \textbf{and} \(r_j<\tau_{\mathrm{stop}}\)}
                \STATE mode \(\gets\textsc{FreeRun}\).
            \ENDIF
        \ENDIF
    \ENDIF
    \IF{mode \(=\textsc{FreeRun}\)}
        \FOR{\textbf{each unfinished branch} \(i\in[K]\) \textbf{in parallel}}
            \STATE Continue decoding branch \(i\) until EOS or \(L_{\max}\) is reached, without further synchronization.
        \ENDFOR
        \STATE \textbf{break}
    \ENDIF
\ENDWHILE
\STATE \textbf{return} \(\{h_i\}_{i=1}^{K}\).
\end{algorithmic}
\end{algorithm*}





\section{Details of Preliminary Studies}
\label{app:preliminary_studies}

\subsection{Information Statistics during Parallel Search}
\label{app:info_statistics}

We provide details for the information-statistics experiment in Figure~\ref{fig:gain_duplication}.
The experiment is conducted on HMMT24--25 using \textsc{Qwen3-30B-A3B-Thinking-2507} with \(64\) parallel reasoning branches, where each search step corresponds to a \(1024\)-token generation chunk.
At each step, we use the same information-extraction prompt as CPT, detailed in Appendix~\ref{app:prompts}, to extract compact information units from branch trajectories.
For each problem, we maintain a running analysis pool and compare each extracted candidate with existing pool entries using the same embedding-based semantic de-duplication protocol as CPT: if its maximum similarity to existing entries is below the de-duplication threshold, it is admitted into the pool and counted as newly discovered information for the current step; otherwise, it is counted as duplicate information for the current step.
Since one search step consumes \(64\times1024\) generated tokens, we normalize the raw information counts and report the number of information units per \(10{,}240\) generated tokens, i.e., multiplying each per-step count by \(10{,}240/(64\times1024)\).

\subsection{Offline Information Injection}
\label{app:info_injection}

We provide details for the offline information-injection experiment in Table~\ref{tab:info_injection_4b}, which tests whether information originally scattered across different reasoning branches is decision-useful once shared. For each problem, we first collect parallel reasoning trajectories, use the same information-extraction prompt as CPT (Appendix~\ref{app:prompts}) to extract compact information units, and merge them into a deduplicated information pool using the same semantic de-duplication protocol. Given an injection ratio \(\rho\in\{0,20,40,60,80,100\}\), we randomly sample \(\rho\%\) of the deduplicated pool entries and inject them into the initial shared-information section of the prompt, where \(0\%\) denotes no injected information and \(100\%\) denotes using the full pool. We then run standard parallel sampling with \(64\) branches on HMMT24--25 using \textsc{Qwen3-4B-Thinking-2507}, without online search-time extraction, pool updating, or broadcasting. All branches receive the same injected context at the beginning, and we report Pass@1, generated tokens, and wall-clock latency. All injected information is extracted from model-generated trajectories, without using gold answers, external verifiers, or human annotations. We repeat each setting for \(8\) independent runs and report the averaged results.

\begin{figure*}[t]
\centering
\includegraphics[width=1.0\textwidth]{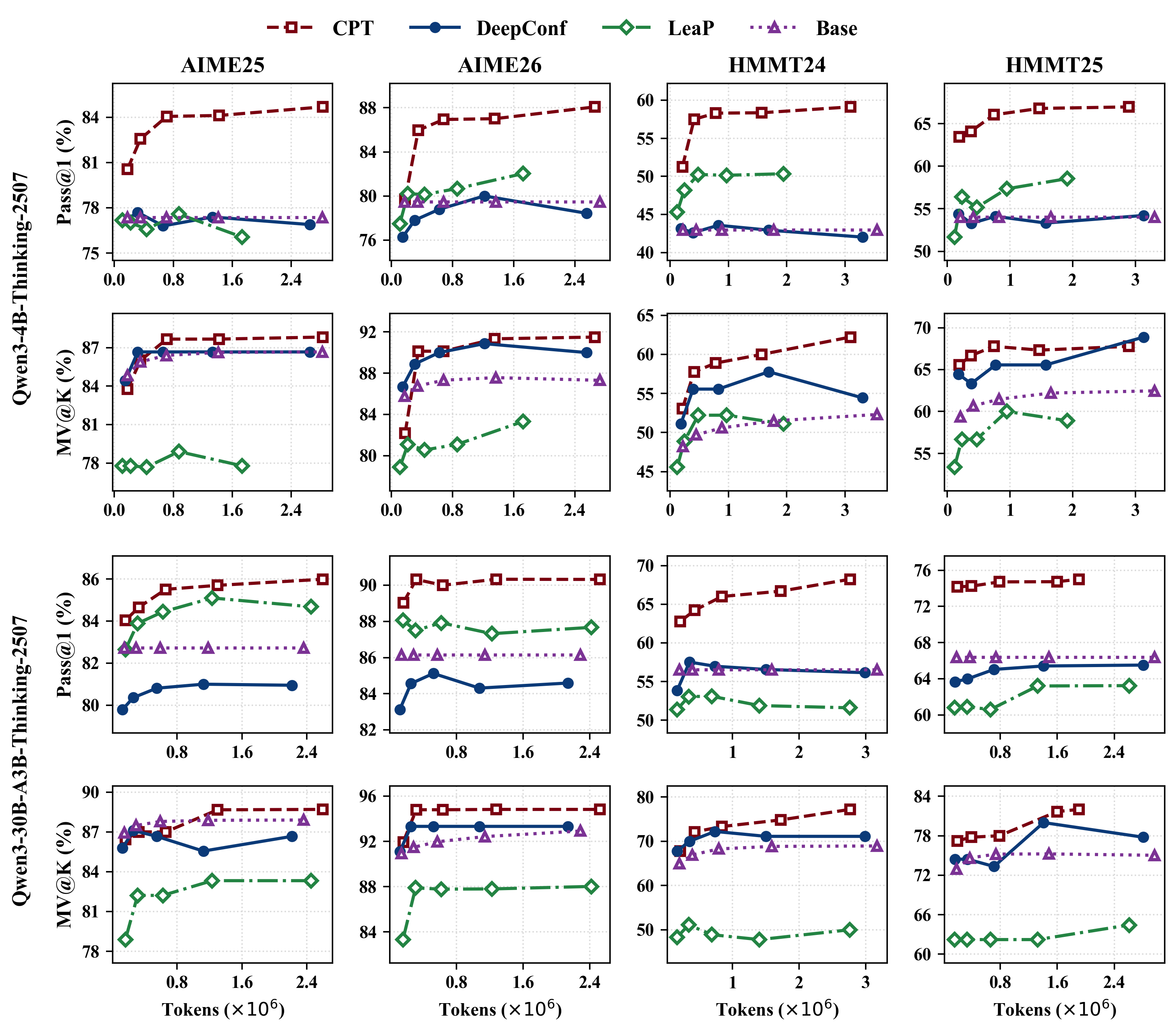}
\caption{Accuracy--Tokens comparison across models and benchmarks under different rollout budgets.
}
\label{fig:main_results_tokens}
\end{figure*}

\begin{figure*}[t]
\centering
\includegraphics[width=1.0\textwidth]{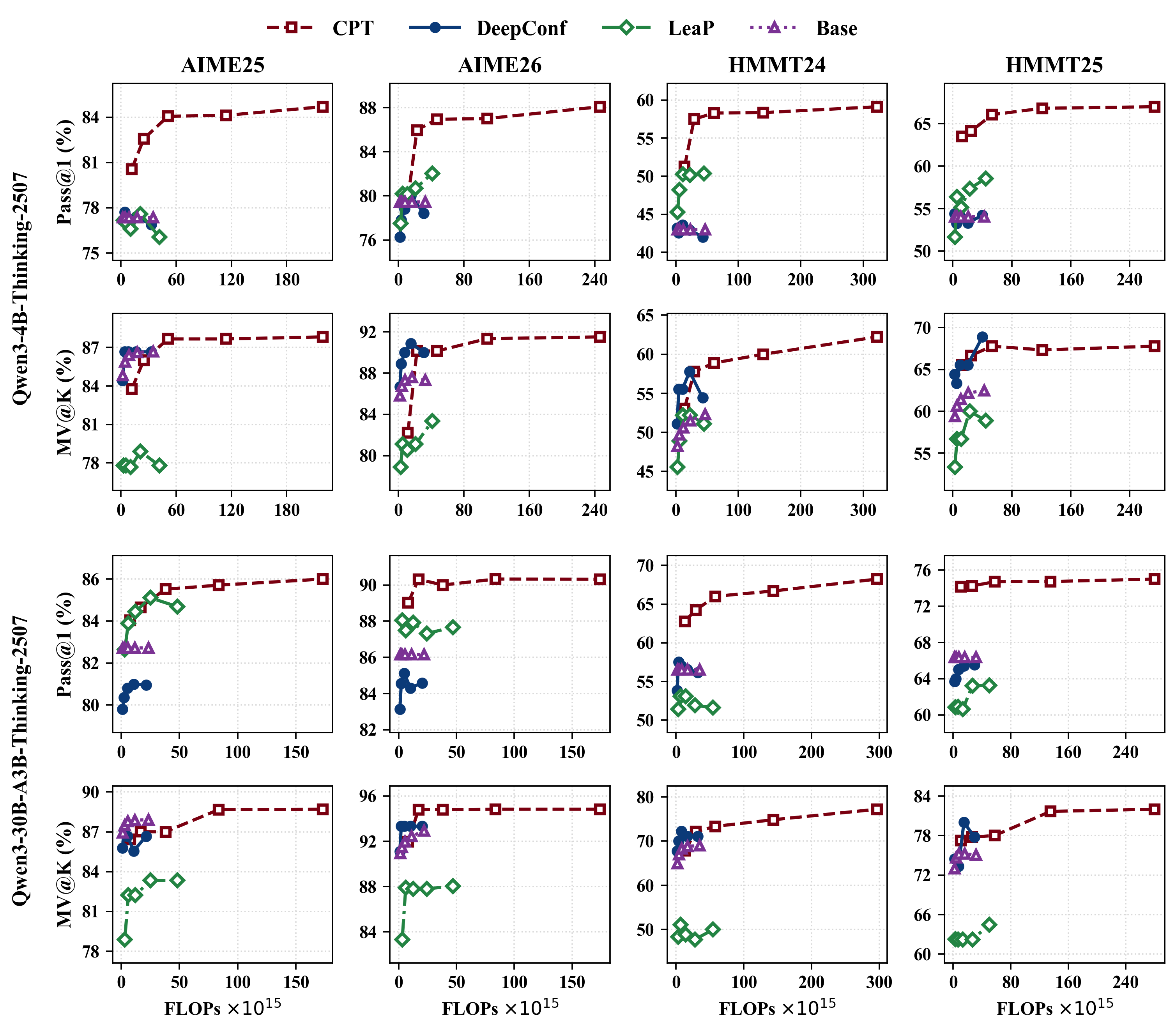}
\caption{Accuracy--FLOPs comparison across models and benchmarks under different rollout budgets. See Appendix~\ref{app:flops_estimation} for details on FLOPs estimation.
}
\label{fig:main_results_flops}
\end{figure*}

\section{Analysis of Generated Tokens and FLOPs}
\label{app:flops_tokens_analysis}

\subsection{Additional Accuracy--Cost Results}
\label{app:additional_cost_results}

We provide supplementary cost-efficiency results under generated-token and FLOPs metrics.
The token-based comparison (Figure~\ref{fig:main_results_tokens}) complements our latency evaluation by measuring the amount of generated computation.
CPT generally achieves a stronger accuracy--tokens trade-off, suggesting that search-time information sharing reduces redundant decoding caused by repeated rediscovery across parallel branches.
In contrast, the FLOPs comparison (Figure~\ref{fig:main_results_flops}) reflects the additional prefilling computation introduced by prompt-level broadcasts.
Since updating the shared-information section changes the input context, standard decoding implementations may need to recompute the prefix, making CPT less favorable under a pure FLOPs budget.
These results provide a fuller view of the computational trade-offs of CPT: it improves latency-aware and token-efficient parallel exploration, while FLOPs efficiency remains an important direction for future cache-aware implementations.

\subsection{FLOPs Estimation}
\label{app:flops_estimation}

To explicitly quantify the computational efficiency of different inference strategies, we estimate the Floating Point Operations (FLOPs) following the methodology used in the NVIDIA NeMo framework~\citep{nemo-rl}. 

\paragraph{Unified FLOPs Estimator.}
We decompose an inference process into a set of model requests, denoted by $\mathcal{R}$. For each request $r \in \mathcal{R}$, we define $S_r$ as the effective linear token term, which is used to estimate the cost of linear projections, MLP layers, and vocabulary projection. We further define $Q_r$ as the effective quadratic attention term, which captures the cost of causal self-attention. The aggregated statistics are:
\begin{align}
    S &= \sum_{r \in \mathcal{R}} S_r, \\
    Q &= \sum_{r \in \mathcal{R}} Q_r.
\end{align}

Let $h$ denote the hidden size, $N_{\ell}$ the number of Transformer layers, $A$ the number of attention heads, $G$ the number of key-value heads, $V$ the vocabulary size, and $d_{\mathrm{ff}}^{\mathrm{active}}$ the active FFN intermediate dimension. For dense models, $d_{\mathrm{ff}}^{\mathrm{active}}$ is the standard FFN intermediate dimension; for MoE models, it is multiplied by the number of activated experts per token.

The total FLOPs are decomposed into attention, MLP, and vocabulary projection costs:
\begin{align}
    \mathrm{FLOPs}_{\mathrm{attn}}
    &=
    2N_{\ell}h^2
    \left[
        \left(2\frac{G}{A}+2\right)S
        +
        \frac{Q}{h}
    \right], \\
    \mathrm{FLOPs}_{\mathrm{mlp}}
    &=
    2N_{\ell}h
    \left(3d_{\mathrm{ff}}^{\mathrm{active}}\right)S, \\
    \mathrm{FLOPs}_{\mathrm{vocab}}
    &=
    2hVS.
\end{align}
Therefore, the total computational cost is:
\begin{equation}
\begin{aligned}
\mathrm{FLOPs}_{\mathrm{total}}
&=
\mathrm{FLOPs}_{\mathrm{attn}}
+
\mathrm{FLOPs}_{\mathrm{mlp}} \\
&\quad+
\mathrm{FLOPs}_{\mathrm{vocab}} .
\end{aligned}
\end{equation}

\paragraph{Full-prefill Generation.}
For a standard independent generation request, let $p_r$ denote the number of prompt tokens and $o_r$ denote the number of generated output tokens. Since each request processes the full prompt before autoregressive decoding, we define:
\begin{align}
    S_r &= p_r + o_r, \\
    Q_r &= (p_r + o_r)^2.
\end{align}

\paragraph{CPT.}
CPT consists of two types of computation: normal reasoning generation and blackboard information extraction. Let $\mathcal{R}_{\mathrm{gen}}$ denote the set of normal reasoning requests, and $\mathcal{R}_{\mathrm{ext}}$ denote the set of information extraction requests.

For a normal reasoning request $r \in \mathcal{R}_{\mathrm{gen}}$, let $L_r$ be its effective sequence length. We use:
\begin{align}
    S_r &= L_r, \\
    Q_r &= L_r^2.
\end{align}

For an information extraction request $r \in \mathcal{R}_{\mathrm{ext}}$, we estimate its cost under a cached-decoding assumption. Let $p_r$ be the length of the extraction context and $o_r$ be the number of tokens generated during extraction. Since the context has already been cached, the linear term only counts newly generated tokens:
\begin{align}
    S_r &= o_r.
\end{align}
The attention term is:
\begin{align}
    Q_r &= 2p_ro_r + o_r^2.
\end{align}

Thus, the overall token-dependent statistics for CPT are:
\begin{align}
    S_{\mathrm{CPT}}
    &=
    \sum_{r \in \mathcal{R}_{\mathrm{gen}}} L_r
    +
    \sum_{r \in \mathcal{R}_{\mathrm{ext}}} o_r, \\
    Q_{\mathrm{CPT}}
    &=
    \sum_{r \in \mathcal{R}_{\mathrm{gen}}} L_r^2
    +
    \sum_{r \in \mathcal{R}_{\mathrm{ext}}}
    \left(2p_ro_r + o_r^2\right).
\end{align}

\paragraph{LeaP.}
LeaP performs inference through multiple consecutive segments, including normal reasoning, intermediate summarization, comment insertion, and final-answer generation. Since these segments are generated sequentially within the same trajectory, we estimate its cost using an incremental KV-cache formulation.

For segment $r$, let $C_r$ denote the number of tokens cached before this segment, and let
\begin{equation}
    x_r = p_r + o_r
\end{equation}
denote the number of newly processed tokens in the current segment, where $p_r$ is the newly inserted prompt length and $o_r$ is the generated output length. The effective statistics are:
\begin{align}
    S_r &= x_r, \\
    Q_r &= x_r(2C_r + x_r).
\end{align}

Therefore, the overall statistics for LeaP are:
\begin{align}
    S_{\mathrm{LeaP}}
    &=
    \sum_{r \in \mathcal{R}_{\mathrm{LeaP}}} x_r, \\
    Q_{\mathrm{LeaP}}
    &=
    \sum_{r \in \mathcal{R}_{\mathrm{LeaP}}}
    x_r(2C_r+x_r).
\end{align}

\paragraph{DeepConf.}
DeepConf is treated as a set of independent sampled completions. For problem $q$, let $p_q$ denote the prompt length after applying the chat template, and let $o_{q,i}$ denote the generated length of the $i$-th completion. Since each completion is sampled independently, each one is counted as a full-prefill request:
\begin{align}
    S_{q,i} &= p_q + o_{q,i}, \\
    Q_{q,i} &= (p_q + o_{q,i})^2.
\end{align}

The overall statistics for DeepConf are:
\begin{align}
    S_{\mathrm{DeepConf}}
    &=
    \sum_q \sum_i
    \left(p_q + o_{q,i}\right), \\
    Q_{\mathrm{DeepConf}}
    &=
    \sum_q \sum_i
    \left(p_q + o_{q,i}\right)^2.
\end{align}

Finally, for each method, the corresponding $S$ and $Q$ are substituted into the unified FLOPs estimator to obtain the total computational cost.

\section{Theory: Collaborative Sharing Improves Global Decision Information}
\label{sec:theory}

We provide an information-theoretic analysis of collaborative parallel thinking.
The analysis abstracts away implementation details such as when to broadcast or how many items to broadcast, and focuses on the core question:
how does sharing intermediate search information affect the amount of \emph{global decision information} obtained from parallel exploration?

Our main observation is that parallel reasoning paths may produce substantial path-wise information, but the information useful to the final decision is not simply the sum of path-wise information across paths.
Different paths may rediscover the same conclusions, constraints, or failure patterns, causing overlapping information to be counted multiple times in isolated exploration.
Collaborative sharing can improve test-time scaling efficiency by pooling path-private information into a shared context, reducing redundant local information, and converting it into globally useful decision information.

\subsection{Decision information in parallel reasoning}

Fix a query \(q\), and let \(Y\) denote the target answer under an epistemic distribution over candidate answers.
Although the ground-truth answer is deterministic, the solver is uncertain about it before reasoning; our information-theoretic quantities are defined with respect to this uncertainty.
At some stage of inference, let \(S\) denote the information state already available to the system.
Suppose \(K\) parallel reasoning paths continue exploration and produce extracted search information
\[
Z_{1:K}=(Z_1,\dots,Z_K),
\]
where \(Z_i\) denotes the information extracted from path \(i\), such as intermediate conclusions, constraints, counterexamples, or failure patterns.
We interpret \(Z_i\) as residual search information beyond the current state \(S\).

We measure the decision value of information by its mutual information with the target answer.
All entropy and mutual-information quantities below are assumed to be finite; the Gaussian example later uses the continuous analogue.

\begin{definition}[Pooled global decision information]
Given the current information state \(S\), the pooled global decision information gained from the next parallel exploration step is
\[
\mathcal{G}(Z_{1:K}\mid S,q)
:=
I(Y;Z_{1:K}\mid S,q).
\]
This quantity measures the decision information obtained after pooling the extracted information from all paths into a shared context.
\end{definition}

\begin{definition}[Aggregate path-wise information]
The aggregate path-wise information produced by the \(K\) paths is
\[
\mathcal{L}(Z_{1:K}\mid S,q)
:=
\sum_{i=1}^K I(Y;Z_i\mid S,q).
\]
This quantity counts the decision information produced by each path separately, so overlapping information across paths is counted multiple times.
\end{definition}

The quantity \(\mathcal{L}\) measures how much decision-relevant information each path appears to produce in isolation.
However, it can overestimate the amount of new information available to the global decision process, because different paths may produce overlapping information.

\begin{definition}[Redundant local information]
We define the redundant local information as
\[
\mathcal{R}(Z_{1:K}\mid S,q)
:=
\mathcal{L}(Z_{1:K}\mid S,q)
-
\mathcal{G}(Z_{1:K}\mid S,q).
\]
Intuitively, \(\mathcal{R}\) measures the amount of path-wise decision information that is repeatedly counted across isolated paths but does not translate into additional pooled global decision information.
\end{definition}

\begin{assumption}[Redundancy-dominated extracted information]
\label{ass:redundancy_dominated}
For the extracted search information considered in this analysis, pooling information across paths does not create more decision information than the sum of its path-wise contributions:
\[
\mathcal{G}(Z_{1:K}\mid S,q)
\le
\mathcal{L}(Z_{1:K}\mid S,q).
\]
Equivalently,
\[
\mathcal{R}(Z_{1:K}\mid S,q)\ge 0.
\]
This assumption abstracts the setting where extracted information units are self-contained intermediate conclusions, constraints, counterexamples, or failure patterns, and the main inefficiency comes from repeated or overlapping information across paths.
\end{assumption}

\paragraph{Decision value under log-loss.}
The mutual-information objective has a standard decision-theoretic interpretation.
If the final predictor is evaluated by log-loss, then the optimal expected loss before observing additional information \(A\) is
\[
H(Y\mid S,q),
\]
whereas the optimal expected loss after observing \(A\) is
\[
H(Y\mid A,S,q).
\]
Therefore, the expected improvement from observing \(A\) is
\[
H(Y\mid S,q)-H(Y\mid A,S,q)
=
I(Y;A\mid S,q).
\]
Thus, maximizing pooled global decision information is equivalent to maximizing the expected log-loss reduction of an optimal Bayesian decision rule.

\subsection{From path-wise information to pooled global information}

The definitions above imply a simple but important decomposition:
\[
\mathcal{L}(Z_{1:K}\mid S,q)
=
\mathcal{G}(Z_{1:K}\mid S,q)
+
\mathcal{R}(Z_{1:K}\mid S,q).
\]
Equivalently,
\[
\mathcal{G}(Z_{1:K}\mid S,q)
=
\mathcal{L}(Z_{1:K}\mid S,q)
-
\mathcal{R}(Z_{1:K}\mid S,q).
\]
In words, aggregate path-wise information equals pooled global decision information plus redundant local information.
This decomposition captures the central inefficiency of independent parallel exploration:
even if each path produces useful local information, repeated discoveries across paths can inflate the path-wise information total without increasing the pooled global decision information.

We next relate redundant local information to a standard information-theoretic dependence measure.

\begin{definition}[Conditional total correlation]
For random variables \(Z_{1:K}\) and conditioning variable \(A\), the conditional total correlation is
\[
\mathrm{TC}(Z_{1:K}\mid A)
:=
\sum_{i=1}^K H(Z_i\mid A)
-
H(Z_{1:K}\mid A).
\]
It measures statistical dependence among \(Z_1,\dots,Z_K\) given \(A\).
\end{definition}

\begin{proposition}[Redundant information and total correlation]
\label{prop:tc_redundancy}
For any joint distribution over \(Y,Z_{1:K},S\), conditioned on the fixed query \(q\), we have
\[
\begin{aligned}
\mathcal{R}(Z_{1:K}\mid S,q)
&=
\mathrm{TC}(Z_{1:K}\mid S,q)\\
&\quad -
\mathrm{TC}(Z_{1:K}\mid Y,S,q).
\end{aligned}
\]
In particular, under the clean model where \(Z_1,\dots,Z_K\) are mutually conditionally independent given \(Y,S,q\), we have
\[
\mathcal{R}(Z_{1:K}\mid S,q)
=
\mathrm{TC}(Z_{1:K}\mid S,q)
\ge 0.
\]
\end{proposition}

\begin{proof}
Let \(A=(S,q)\), \(B=(Y,S,q)\), and write \(\sum_i\) for \(\sum_{i=1}^K\).
By the definition of \(\mathcal{R}\),
\[
\begin{aligned}
\mathcal{R}
&=
\sum_i I(Y;Z_i\mid A)
-
I(Y;Z_{1:K}\mid A)\\
&=
\sum_i H(Z_i\mid A)
-
H(Z_{1:K}\mid A)\\
&\quad -
\sum_i H(Z_i\mid B)
+
H(Z_{1:K}\mid B)\\
&=
\mathrm{TC}(Z_{1:K}\mid A)
-
\mathrm{TC}(Z_{1:K}\mid B).
\end{aligned}
\]
Restoring \(A=(S,q)\) and \(B=(Y,S,q)\) gives
\[
\begin{aligned}
\mathcal{R}(Z_{1:K}\mid S,q)
&=
\mathrm{TC}(Z_{1:K}\mid S,q)\\
&\quad -
\mathrm{TC}(Z_{1:K}\mid Y,S,q).
\end{aligned}
\]
If \(Z_1,\dots,Z_K\) are mutually conditionally independent given \(Y,S,q\), then
\[
H(Z_{1:K}\mid Y,S,q)
=
\sum_{i=1}^K H(Z_i\mid Y,S,q),
\]
so \(\mathrm{TC}(Z_{1:K}\mid Y,S,q)=0\). Hence
\[
\mathcal{R}(Z_{1:K}\mid S,q)
=
\mathrm{TC}(Z_{1:K}\mid S,q)
\ge 0,
\]
since total correlation is nonnegative.
\end{proof}

Proposition~\ref{prop:tc_redundancy} shows that, under a clean conditional-independence model, redundant local information is exactly the total correlation among path-level information variables.
Thus, when different paths repeatedly discover similar information, the aggregate path-wise information can be large while the pooled global decision information remains limited.

\subsection{When does collaborative sharing help?}

We now compare independent parallel exploration with collaborative exploration.
Let
\[
Z^{\mathrm{ind}}_{1:K}
\]
denote the information variables produced by independent parallel paths, and let
\[
Z^{\mathrm{col}}_{1:K}
\]
denote those produced under collaborative sharing.
Both are conditioned on the same current state \(S\) and query \(q\).

Define
\[
\begin{aligned}
\mathcal{G}_{\mathrm{ind}}
&:=
\mathcal{G}(Z^{\mathrm{ind}}_{1:K}\mid S,q),\\
\mathcal{G}_{\mathrm{col}}
&:=
\mathcal{G}(Z^{\mathrm{col}}_{1:K}\mid S,q),
\end{aligned}
\]
and define
\(\mathcal{L}_{\mathrm{ind}},\mathcal{L}_{\mathrm{col}}\)
and
\(\mathcal{R}_{\mathrm{ind}},\mathcal{R}_{\mathrm{col}}\)
analogously.

The following result is a decomposition rather than an unconditional dominance guarantee.

\begin{theorem}[Collaborative gain decomposition]
\label{thm:collab_gain}
The difference in pooled global decision information between collaborative and independent exploration satisfies
\[
\begin{aligned}
\mathcal{G}_{\mathrm{col}}
-
\mathcal{G}_{\mathrm{ind}}
&=
\left(
\mathcal{L}_{\mathrm{col}}
-
\mathcal{L}_{\mathrm{ind}}
\right)\\
&\quad+
\left(
\mathcal{R}_{\mathrm{ind}}
-
\mathcal{R}_{\mathrm{col}}
\right).
\end{aligned}
\]
Consequently, collaborative sharing improves pooled global decision information whenever
\[
\mathcal{R}_{\mathrm{ind}}
-
\mathcal{R}_{\mathrm{col}}
\ge
\mathcal{L}_{\mathrm{ind}}
-
\mathcal{L}_{\mathrm{col}}.
\]
In particular, if collaboration reduces redundant local information without reducing aggregate path-wise information, then
\[
\mathcal{G}_{\mathrm{col}}
\ge
\mathcal{G}_{\mathrm{ind}}.
\]
\end{theorem}

\begin{proof}
By the path-wise to pooled decomposition,
\[
\mathcal{G}_{\mathrm{col}}
=
\mathcal{L}_{\mathrm{col}}
-
\mathcal{R}_{\mathrm{col}},
\qquad
\mathcal{G}_{\mathrm{ind}}
=
\mathcal{L}_{\mathrm{ind}}
-
\mathcal{R}_{\mathrm{ind}}.
\]
Subtracting the two identities gives
\[
\begin{aligned}
\mathcal{G}_{\mathrm{col}}
-
\mathcal{G}_{\mathrm{ind}}
&=
\left(
\mathcal{L}_{\mathrm{col}}
-
\mathcal{R}_{\mathrm{col}}
\right)
-
\left(
\mathcal{L}_{\mathrm{ind}}
-
\mathcal{R}_{\mathrm{ind}}
\right)\\
&=
\left(
\mathcal{L}_{\mathrm{col}}
-
\mathcal{L}_{\mathrm{ind}}
\right)
+
\left(
\mathcal{R}_{\mathrm{ind}}
-
\mathcal{R}_{\mathrm{col}}
\right).
\end{aligned}
\]
Thus \(\mathcal{G}_{\mathrm{col}}\ge \mathcal{G}_{\mathrm{ind}}\) whenever
\[
\mathcal{R}_{\mathrm{ind}}
-
\mathcal{R}_{\mathrm{col}}
\ge
\mathcal{L}_{\mathrm{ind}}
-
\mathcal{L}_{\mathrm{col}}.
\]
\end{proof}

Theorem~\ref{thm:collab_gain} formalizes the core intuition behind collaborative parallel thinking.
Sharing can help not merely by producing more path-wise information, but by reducing redundant local information that prevents path-wise information from becoming pooled global decision information.
The theorem also makes clear that sharing is not unconditionally beneficial: if shared context reduces path-level exploration too much, the decrease in aggregate path-wise information may offset the reduction in redundant local information.

\subsection{A Gaussian redundancy model}

The previous results are distribution-free identities.
We now give a stylized Gaussian model to illustrate how cross-path dependence reduces the effective information width of parallel exploration.
This model is not intended as a literal model of LLM reasoning; rather, it isolates the effect of correlated path-level information on global decision information.

For simplicity, omit conditioning on \(S,q\).
Let
\[
Y\sim \mathcal{N}(0,1)
\]
be a scalar latent decision signal.
Each path produces a noisy residual information signal
\[
Z_i = Y + U + V_i,\qquad i=1,\dots,K,
\]
where
\[
U\sim \mathcal{N}(0,\sigma_c^2)
\]
is a common noise component shared across paths, and
\[
V_i\sim \mathcal{N}(0,\sigma_u^2)
\]
is path-specific noise.
Assume \(U,V_1,\dots,V_K\) are mutually independent and independent of \(Y\).
Assume \(\sigma_u^2>0\) and \(\sigma_c^2\ge 0\); the degenerate case \(\sigma_u^2=0\) is understood by continuity.

The common noise \(U\) represents shared uncertainty or common bias induced by repeated derivations, similar failure modes, or correlated search behavior across paths.
The path-specific noise \(V_i\) represents idiosyncratic uncertainty that can be averaged out across independent paths.

Let
\[
\sigma^2 := \sigma_u^2+\sigma_c^2,
\qquad
\rho := \frac{\sigma_c^2}{\sigma_u^2+\sigma_c^2}.
\]
Then \(\sigma^2\) is the marginal noise variance of each path, and \(\rho\) is the pairwise noise correlation.

\begin{proposition}[Effective parallel width under correlated search information]
\label{prop:gaussian_eff_width}
Under the Gaussian redundancy model,
\[
I(Y;Z_{1:K})
=
\frac{1}{2}
\log
\left(
1+
\frac{K}{\sigma_u^2+K\sigma_c^2}
\right).
\]
Equivalently,
\[
I(Y;Z_{1:K})
=
\frac{1}{2}
\log
\left(
1+
\frac{K_{\mathrm{eff}}}{\sigma^2}
\right),
\]
where
\[
K_{\mathrm{eff}}
:=
\frac{K}{1+(K-1)\rho}.
\]
\end{proposition}

\begin{proof}
Let \(Z=(Z_1,\dots,Z_K)^\top\) and let \(\mathbf{1}\in\mathbb{R}^K\) be the all-one vector.
The observation model can be written as
\[
Z=\mathbf{1}Y+\varepsilon,
\]
where
\[
\varepsilon = U\mathbf{1} + (V_1,\dots,V_K)^\top
\]
is Gaussian noise with covariance
\[
\Sigma
=
\sigma_u^2 I_K
+
\sigma_c^2 \mathbf{1}\mathbf{1}^\top.
\]
For a scalar Gaussian signal observed through a linear Gaussian channel,
\[
I(Y;Z)
=
\frac{1}{2}
\log
\left(
1+
\mathbf{1}^\top \Sigma^{-1}\mathbf{1}
\right).
\]
By the Sherman--Morrison formula,
\[
\begin{aligned}
\Sigma^{-1}
&=
\frac{1}{\sigma_u^2}I_K
 -
\frac{\sigma_c^2}
{\sigma_u^2(\sigma_u^2+K\sigma_c^2)}
\mathbf{1}\mathbf{1}^\top.
\end{aligned}
\]
Therefore,
\[
\mathbf{1}^\top\Sigma^{-1}\mathbf{1}
=
\tfrac{K}{\sigma_u^2}
-
\tfrac{\sigma_c^2 K^2}{\sigma_u^2(\sigma_u^2+K\sigma_c^2)}
=
\tfrac{K}{\sigma_u^2+K\sigma_c^2}.
\]
Thus,
\[
I(Y;Z_{1:K})
=
\frac{1}{2}
\log
\left(
1+
\frac{K}{\sigma_u^2+K\sigma_c^2}
\right).
\]
Using
\[
\sigma^2=\sigma_u^2+\sigma_c^2,
\qquad
\rho=\frac{\sigma_c^2}{\sigma^2},
\]
we have
\[
\sigma_u^2+K\sigma_c^2
=
\sigma^2\bigl(1+(K-1)\rho\bigr).
\]
Hence
\[
\frac{K}{\sigma_u^2+K\sigma_c^2}
=
\frac{1}{\sigma^2}
\cdot
\frac{K}{1+(K-1)\rho}.
\]
Defining
\[
K_{\mathrm{eff}}
:=
\frac{K}{1+(K-1)\rho}
\]
gives the claimed expression.
\end{proof}

Proposition~\ref{prop:gaussian_eff_width} shows that nominal parallel width \(K\) can be much larger than the effective information width \(K_{\mathrm{eff}}\).
When \(\rho=0\), the path-level noises are independent and
\[
K_{\mathrm{eff}}=K.
\]
In this case, all paths contribute complementary observations.
When \(\rho\to 1\),
\[
K_{\mathrm{eff}}\to 1,
\]
meaning that even many parallel paths behave like a single effective information source.
More generally, for fixed \(\rho>0\),
\[
\lim_{K\to\infty}K_{\mathrm{eff}}=\frac{1}{\rho},
\]
so increasing the number of paths has diminishing information returns when cross-path dependence is high.

This closed-form model illustrates why independent parallel sampling can be inefficient despite using many rollouts.
If different paths produce highly correlated residual search information, increasing \(K\) does not linearly increase global decision information.
In this model, collaboration is useful when it reduces the correlation of residual information that remains to be discovered after conditioning on the shared state.
This does not require raw generations to become more diverse; rather, already shared conclusions are removed from the residual search problem, so subsequent paths can spend less effort rediscovering the same information.

\lstdefinestyle{cptprompt}{
    basicstyle=\ttfamily\scriptsize,
    backgroundcolor=\color{gray!5},
    frame=single,
    rulecolor=\color{black},
    framerule=0.8pt,
    framesep=8pt,
    breaklines=true,
    breakatwhitespace=false,
    columns=fullflexible,
    keepspaces=true,
    showstringspaces=false,
    tabsize=2,
    aboveskip=0pt,
    belowskip=0pt
}

\section{Prompts}
\label{app:prompts}

We list the prompt templates used by our CPT implementation.
The worker prompt is used for each reasoning branch, while the blackboard-write prompt is used by the same policy model to distill reusable intermediate notes from partial branch transcripts.


\begin{figure*}[t!]
    \centering
    \begin{tcblisting}{
        colback=gray!5,
        colframe=black,
        boxrule=0.8pt,
        arc=2pt,
        left=10pt, right=10pt, top=10pt, bottom=10pt,
        title=\textbf{Default Mathematical Answer Prompt},
        listing only,
        listing options={
            basicstyle=\ttfamily\scriptsize,
            breaklines=true,
            columns=fullflexible
        }
    }
Please reason step by step, and put your final answer within \boxed{}.
    \end{tcblisting}
    \caption{\textbf{Default mathematical answer prompt.} This instruction is appended to the worker prompt in the default configuration.}
    \label{fig:default_math_prompt}
\end{figure*}

\begin{figure*}[t!]
    \centering
    \begin{tcblisting}{
        colback=gray!5,
        colframe=black,
        boxrule=0.8pt,
        arc=2pt,
        left=10pt, right=10pt, top=10pt, bottom=10pt,
        title=\textbf{CPT Worker Prompt},
        listing only,
        listing options={
            basicstyle=\ttfamily\scriptsize,
            breaklines=true,
            columns=fullflexible
        }
    }
You are an intelligent reasoning agent solving complex problems step-by-step.

You may occasionally receive external information in the format:
[BLACKBOARD BROADCAST]
...
[/BLACKBOARD BROADCAST]

The blackboard may contain two kinds of reusable intermediate notes:
- insight: potentially useful intermediate facts, relations, reductions, invariants, or local observations.
- pitfall: warnings about possible reasoning errors, unsafe operations, missing cases, or dead ends.

Rules:
1) Blackboard content is NOT part of the original problem statement; treat it only as optional intermediate guidance. It may help you adjust direction, notice useful relations, or avoid repeated mistakes, but it should never replace your own reasoning from the problem statement.

2) Do NOT blindly trust or copy any blackboard note. Treat insights as structural hypotheses rather than proven facts, and use them only after checking their conditions against the problem statement and your own derivation.

3) Be especially skeptical of numerical claims, overly strong claims, uniqueness claims, impossibility claims, or any note that looks like a direct conclusion rather than an intermediate reasoning aid. Do not rely on such content unless you independently derive it.

4) Treat pitfalls as warning signs, not absolute prohibitions. If a pitfall is relevant to your current path, slow down and check the missing condition, unsafe operation, failed assumption, or ignored case before deciding whether to continue or change direction.

If the blackboard conflicts with your current reasoning, re-check the disputed assumption or derivation instead of simply following either side. Maintain independent reasoning diversity: the blackboard should assist your reasoning, not force your reasoning to align with it.
    \end{tcblisting}
    \caption{\textbf{CPT worker prompt.} Each reasoning branch receives this instruction, together with the default mathematical answer prompt in Figure~\ref{fig:default_math_prompt}.}
    \label{fig:cpt_worker_prompt}
\end{figure*}





\begin{figure*}[ht]
    \centering
    \begin{tcblisting}{
        colback=gray!5,
        colframe=black,
        boxrule=0.8pt,
        arc=2pt,
        left=10pt, right=10pt, top=10pt, bottom=10pt,
        title=\textbf{Worker Input Serialization Template},
        listing only,
        listing options={
            basicstyle=\ttfamily\scriptsize,
            breaklines=true,
            columns=fullflexible
        }
    }
[system]
[BLACKBOARD BROADCAST]
{{ selected_blackboard_notes_or_empty }}
[/BLACKBOARD BROADCAST]

[user]
{{ cpt_worker_prompt }}

[user]
{{ original_problem }}

[assistant]
{{ branch_reasoning_continuation }}
    \end{tcblisting}
    \caption{\textbf{Worker input serialization template.} The blackboard broadcast occupies the leading system message so that shared notes always appear before the branch continues its private reasoning trace.}
    \label{fig:cpt_worker_serialization}
\end{figure*}

\begin{figure*}[t!]
    \centering
    \begin{tcblisting}{
        colback=gray!5,
        colframe=black,
        boxrule=0.8pt,
        arc=2pt,
        left=10pt, right=10pt, top=10pt, bottom=10pt,
        title=\textbf{Blackboard-Write System Prompt},
        listing only,
        listing options={
            basicstyle=\ttfamily\scriptsize,
            breaklines=true,
            columns=fullflexible
        }
    }
You are a strict Strategic Reasoning Distiller for a shared blackboard.

Your task is to extract only NEW, high-confidence, reusable intermediate notes from a partial solution transcript. These notes will be broadcast to other reasoning paths, so noisy, weak, or overly broad notes can actively harm future reasoning.

You may also receive:
[HISTORY BROADCAST]
...
[/HISTORY BROADCAST]

These are notes already extracted in previous rounds.

Core principles:
- Treat the transcript as a partial and possibly incomplete reasoning process. Focus only on intermediate information that remains locally valid and reusable.
- Before writing any note, verify that it is directly supported by the transcript, standard mathematical facts, or a clearly valid local derivation. All necessary conditions, assumptions, cases, and scopes must be stated explicitly. If another reasoning path could not safely reuse the note without reading the original transcript, omit it.
- Be conservative. Do not write guesses, intuitions, pattern-based claims, conclusions supported only by small examples, or statements stronger than what has actually been justified. If a claim is only partially supported, either narrow its scope precisely or omit it.
- Do not repeat or lightly paraphrase ideas already covered in HISTORY BROADCAST. Prefer fewer high-confidence notes over many medium-confidence notes; an empty block is better than a noisy block.

What to write:
- insight = a locally verified, reusable intermediate fact, relation, invariant, formula, reduction, or scoped observation.
- pitfall = a concrete negative constraint or warning, such as a non-equivalent transformation, unsafe division, missing case, circular argument, unjustified generalization, dead-end strategy, or contradiction.
- For final-answer-like claims in the transcript, extract only the reusable intermediate relation, condition, or risk behind them, not the final result, total count, uniqueness claim, or impossibility claim itself.

Output format:
- Output ONLY:
  [BB_WRITE]
  ...
  [/BB_WRITE]
- Each bullet must start with:
  - (type=insight|pitfall)
- One sentence per bullet.
- Keep the total output concise
- If no note passes the rules, output:
  [BB_WRITE]
  [/BB_WRITE]
    \end{tcblisting}
    \caption{\textbf{Blackboard-write system prompt.} This prompt instructs the model to extract concise, reusable, and conservative notes from a partial branch transcript.}
    \label{fig:cpt_bb_write_system_prompt}
\end{figure*}

\end{document}